\documentclass{article}
\usepackage{iclr2027_conference,times}
\usepackage{amsmath,amsfonts,bm}

\def\eqref#1{equation~\ref{#1}}
\def\1{\bm{1}}

\DeclareMathAlphabet{\mathsfit}{\encodingdefault}{\sfdefault}{m}{sl}
\SetMathAlphabet{\mathsfit}{bold}{\encodingdefault}{\sfdefault}{bx}{n}

\def\sP{{\mathbb{P}}}

\usepackage{amsmath,amssymb,amsthm}
\usepackage{booktabs}
\usepackage{graphicx}
\usepackage{tabularx}
\usepackage{siunitx}
\usepackage{longtable}
\usepackage{float}
\usepackage[hidelinks]{hyperref}
\usepackage{url}
\usepackage{placeins}

\AddToHook{env/table/begin}{\setlength{\belowcaptionskip}{6pt}}
\AddToHook{env/table*/begin}{\setlength{\belowcaptionskip}{6pt}}

\newtheoremstyle{boldplain}
  {\topsep}{\topsep}{\itshape}{}{\bfseries}{.}{5pt plus 1pt minus 1pt}
  {\thmname{#1}\thmnumber{ #2}\thmnote{ (#3)}}
\newtheoremstyle{bolddefinition}
  {\topsep}{\topsep}{}{}{\bfseries}{.}{5pt plus 1pt minus 1pt}
  {\thmname{#1}\thmnumber{ #2}\thmnote{ (#3)}}
\theoremstyle{boldplain}
\newtheorem{theorem}{Theorem}
\newtheorem{proposition}{Proposition}
\newtheorem{corollary}{Corollary}
\theoremstyle{bolddefinition}
\newtheorem{definition}{Definition}
\theoremstyle{boldplain}
\newcommand{\method}{loss-conditioned state execution}
\newcommand{\LCB}{\operatorname{LCB}}

\title{When Should a World Model Move?\\
Loss-Conditioned State Execution}

\author{%
\normalfont
Jintao Xu\textsuperscript{1,}\thanks{Equal contribution.}\quad
Zhengyu Chen\textsuperscript{1,}\footnotemark[1]\quad
Ben Zhang\textsuperscript{1,}\footnotemark[1]\quad
Yongzhi Qi\textsuperscript{1,}\thanks{Corresponding author.}\quad
Jianshen Zhang\textsuperscript{1}\\[3pt]
\textsuperscript{1}Supply Chain Tech Team Y, JD.com\\[3pt]
{\footnotesize\texttt{\{xujintao.3014,chenzhengyu8,zhangben22,qiyongzhi1,zhangjianshen\}@jd.com}}
}

\iclrfinalcopy

\begin{document}
\maketitle
\lhead{}

\begin{abstract}
We introduce loss-conditioned state execution, a model-agnostic method that decides whether to execute
a world model's fixed feasible proposal or retain the current state. Predictive informativeness alone,
however, does not establish whether an update will reduce downstream loss. Occurrence ranking can
approach perfection while persistence remains the unique absolute-loss Bayes action. Two transition laws
can also share occurrence information and conditional variance yet require opposite absolute-loss
decisions. We formalize state movability as the existence of a loss-reducing feasible correction and
distinguish it from the benefit of a particular proposal. Our method constructs a loss-specific feasible
proposal from a predictive distribution and evaluates its groupwise bounded-loss gain over persistence on
independent calibration units. The proposal is executed only in groups with a positive simultaneous lower
confidence bound. For fixed proposals and groups with bounded unit losses, we prove that every accepted
group has lower expected loss than persistence with high probability when calibration units are i.i.d.
draws from the target population. Experiments on public forecasting and action-conditioned dynamics benchmarks show supported
updates and a trade-off between certification and coverage. On
28,684 held-out M4 Monthly series, the method executes the proposal for 14.0\% of series and achieves
bounded loss 0.588, compared with 0.599 for persistence and 0.621 for always executing the proposal. The
paired 95\% bootstrap intervals for both comparisons lie below zero. In constrained forecasting of
six unhealthy-inventory types from JD.com, a leading e-retailer in China, strong occurrence-ranking signal
coexists with a loss-based preference for persistence, illustrating why event predictability and state
execution must be evaluated separately.
\end{abstract}

\section{Introduction}

World models map histories and optional actions to distributions over future states. They are commonly
judged by likelihood, calibration, or downstream control, with rollout length and uncertainty used to
limit unreliable simulation. To select a state prediction, a model must also decide:\\*
\makebox[\linewidth][c]{\emph{Move away from the current state, or persist?}}
\par\vspace{-\parskip}
\noindent This decision is loss dependent. Under absolute loss, persistence is optimal whenever zero change is a
conditional median, whereas squared loss depends on the conditional mean and asymmetric costs select a
task-specific quantile \citep{gneiting2011point}. Consequently, a model can rank change events almost perfectly yet induce no state
update that improves absolute error over persistence. This distinction matters particularly in sticky or sparse dynamics, where
event prediction and loss-reducing state movement are different objectives. A scalar variance estimate cannot generally
resolve the ambiguity: we construct transition
laws with identical occurrence information and conditional variance but opposite loss-optimal movement
decisions. We call a state \emph{movable} under a declared loss when some feasible correction has lower
conditional risk than persistence under the data-generating transition law. This population property is
distinct from the benefit of a particular fitted proposal and the calibration evidence supporting its execution.

In \method, a predictive transition distribution induces a Bayes correction under the
stated loss, which is then mapped into the feasible state set. Training data define interpretable groups before calibration. An independent calibration split
evaluates the \emph{executed proposal}'s bounded-loss gain over persistence. The gate executes proposals only
for groups with a positive simultaneous lower confidence bound. Otherwise, the state prediction equals
the current state.
Figure~\ref{fig:method} shows the resulting proposal--evaluation--fallback interface.

Our main contributions are organized around three aspects:
\begingroup
\setlength{\leftmargini}{1.2em}
\setlength{\labelsep}{0.5em}
\setlength{\labelwidth}{0.7em}
\begin{itemize}
  \item \textbf{Loss-specific state movability.} We distinguish the existence of a loss-reducing feasible
  correction from the benefit of a fixed model proposal. We characterize persistence regions for common
  scalar losses and prove that identical occurrence information and conditional variance can imply
  opposite absolute-loss execution decisions.
  \item \textbf{Proposal benefit and certified execution.} We bound the gap between state movability
  and fixed-proposal benefit in terms of transition-distribution error and feasibility or optimization
  suboptimality. Independent calibration evaluates the executed proposal against persistence and
  certifies positive group-average bounded-loss gain under explicit sampling assumptions.
  \item \textbf{Cross-domain movement regimes.} Across synthetic settings, Monash and M4 forecasting
  benchmarks, Minari FourRooms and MuJoCo, and a constrained six-type unhealthy-inventory case from JD.com,
  we distinguish supported updates, absent fitted-proposal gain, and insufficient calibration evidence.
  The experiments quantify the trade-off between certification, update coverage, and prediction error.
\end{itemize}
\endgroup

\begin{figure}[!h]
  \centering
  \includegraphics[width=\linewidth]{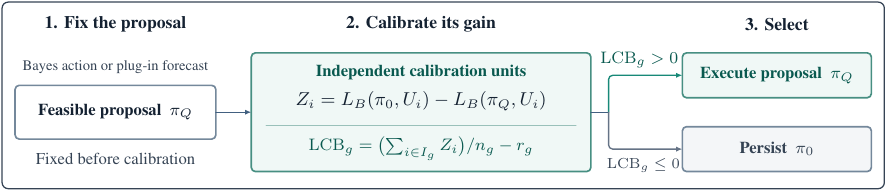}
  \caption{Loss-conditioned state execution. A fixed feasible proposal is executed when independent
  calibration certifies positive groupwise gain over persistence under the declared unit loss.}
  \label{fig:method}
\end{figure}
\FloatBarrier

\section{Related Work}

\paragraph{World models and conservative rollout.}
Learned dynamics support planning by simulating future states. Methods by \citet{ha2018worldmodels},
\citet{chua2018pets}, \citet{janner2019mbpo}, \citet{yu2020mopo}, and \citet{kidambi2020morel} use
probabilistic ensembles, short rollouts, pessimism, or unknown-region penalties to limit accumulated model
error. Uncertainty-aware rollout adaptation governs where or how far learned dynamics are used
\citep{frauenknecht2024macura}. World-action models adapt imagination depth using planning gain and
computation cost \citep{lu2026riseadaptive} or action-chunk length using future--reality consistency
\citep{wang2026trustimagination}. These approaches select the extent of imagination or execution. Our
criterion instead evaluates whether a fixed feasible state correction improves over persistence under
the declared loss, complementing the choice of transition learner and rollout horizon.

\paragraph{Probabilistic and decision-focused prediction.}
\citet{gneiting2007scoring} and \citet{gneiting2014forecasting} describe proper scoring rules that evaluate
predictive distributions without committing to a single-valued decision. Loss-consistent point forecasts target functionals
such as means or quantiles \citep{gneiting2011point}. The area under the receiver operating characteristic curve (AUROC)
summarizes ranking rather than a fixed-cost
decision \citep{hand2009auc}. Conditional predictive-ability methods use observable information and loss
differences to test or select between forecast rules at future origins under possible misspecification
\citep{giacomini2006conditional}. Predictive-to-prescriptive methods connect estimated outcomes to
optimization objectives \citep{bertsimas2020prescriptive}. Our formulation compares a post-feasibility
state correction with persistence under the declared loss, distinguishing the existence
of a beneficial feasible correction (state movability) from the gain of the model-induced correction
(proposal benefit). Independent calibration assesses evidence for executing the fixed proposal.

\paragraph{Selective prediction, deferral, and time-series forecasting.}
\citet{elyaniv2010selective}, \citet{geifman2019selectivenet}, and \citet{noskov2024selective} study
selective prediction, which trades output coverage for conditional risk by withholding uncertain predictions.
Learning to defer instead routes a decision to an external expert \citep{mozannar2020defer}. \citet{tomar2026shapelets} and
\citet{inacio2026selective} reject difficult time-series predictions using learned structural or transfer scores.
Our rejection semantics assign persistence, a feasible baseline prediction, for every rejected proposal.
Coverage is therefore determined by evidence that the fixed candidate improves over persistence under the
declared loss.

\paragraph{Baseline fallback and statistical certification.}
Safe policy improvement bootstraps with a baseline where uncertainty is high \citep{laroche2019spibb}, and
runtime-assurance architectures switch between advanced and verified baseline controllers
\citep{chen2022simplexdrive}. Risk-controlled policy post-processing instead maximizes agreement with a
deterministic baseline subject to a chance-risk budget \citep{joshi2026riskcontrolled}. Certificate-driven
time-series calibration trains an online gated residual over a fixed backbone and uses martingale
PAC-Bayesian certificates under temporal dependence and shift \citep{huang2026certificate}.
\method{} uses independent calibration units and the learn-then-test construction
\citep{angelopoulos2025learntest} to certify groupwise gain over current-state persistence for a
post-feasibility proposal fixed before calibration. Appendix~\ref{app:method-comparison} compares the
decision objects, fallback behavior, and evaluation targets of these methods.

\section{Loss-Conditioned State Execution}

Let $t$ index transition time and $H_t$ the observed history, including any observed action $a_t$. A transition model
with parameters $\theta$ specifies a conditional predictive distribution $Q_\theta(\cdot\mid H_t)$ over the state correction
$\Delta_t:=S_{t+1}-S_t$. Let $\mathcal D(H_t)$ be the set of feasible corrections and assume
$0\in\mathcal D(H_t)$, so persistence is executable. For a specified horizon $h$, the same formulation
uses $\Delta_{t,h}:=S_{t+h}-S_t$ and a fixed horizon-specific proposal rule. All rules are measurable,
and all conditional risks and risk infima below are finite. For a raw correction domain
$\mathcal A(H_t)$ and loss $\ell$, assume the Bayes minimum is attained and define
\begin{equation*}
  b_Q(H_t)\in\arg\min_{c\in\mathcal A(H_t)}
  \mathbb E_{\Delta\sim Q_\theta(\cdot\mid H_t)}[\ell(c,\Delta)].
\end{equation*}
Let $F_{H_t}$ map proposed states into $S_t+\mathcal D(H_t)$ and define the \emph{executed proposal}
$c_Q(H_t):=F_{H_t}(S_t+b_Q(H_t))-S_t$. Direct optimization over $\mathcal D(H_t)$ uses
$F_{H_t}$ as the identity on feasible states. The persistence correction is $c_0(H_t):=0$, and
$P(\cdot\mid H_t)$ denotes the data-generating conditional law of $\Delta_t$.

\begin{definition}[State movability and proposal benefit]
\label{def:movability}
For conditional risk $R_P(c\mid H_t):=\mathbb E_P[\ell(c,\Delta_t)\mid H_t]$, define
\begin{align*}
  M_\ell(P\mid H_t)
  &:=R_P(0\mid H_t)-\inf_{c\in\mathcal D(H_t)}R_P(c\mid H_t),\\
  V_\ell(Q_\theta,P\mid H_t)
  &:=R_P(0\mid H_t)-R_P(c_Q(H_t)\mid H_t).
\end{align*}
The quantity $M_\ell$ measures state movability, and the state is \emph{movable} when $M_\ell>0$.
The quantity $V_\ell$ measures proposal benefit, and the executed proposal has \emph{population support}
when $V_\ell>0$.
\end{definition}

Since $0\in\mathcal D(H_t)$, $M_\ell\geq0$, with equality exactly when persistence is optimal.
Feasibility of $c_Q(H_t)$ gives $V_\ell\leq M_\ell$, so positive proposal benefit implies movability.
The gap $M_\ell-V_\ell$ is the executed proposal's regret relative to the optimal feasible risk.
Independent calibration below assesses group-average proposal benefit under a bounded unit loss.

\subsection{Persistence regions and ranking--movement separation}

Whether persistence is Bayes-optimal depends on the downstream loss. The following result characterizes
its optimality under three common scalar losses.

\begin{proposition}[Bayes persistence under common scalar prediction losses]
\label{prop:common-losses-main}
For a fixed context $H$, let $\Delta$ be its scalar state correction and optimize over $c\in\mathbb R$.
Then:
\smallskip

\noindent\textbf{(i) Squared loss.} If $\mathbb E[\Delta^2\mid H]<\infty$, zero is the unique Bayes
correction if and only if $\mathbb E[\Delta\mid H]=0$.

\smallskip

\noindent\textbf{(ii) Pinball loss.} For $\tau\in(0,1)$, define
$\rho_\tau(\Delta-c):=\tau(\Delta-c)_{+} +(1-\tau)(c-\Delta)_{+}$, where
$(x)_{+}:=\max\{x,0\}$. If $\mathbb E[|\Delta|\mid H]<\infty$, zero is Bayes-optimal if and only if
\begin{equation}
  \sP(\Delta<0\mid H)\leq\tau\leq\sP(\Delta\leq0\mid H).
  \label{eq:zero-quantile}
\end{equation}

\smallskip

\noindent\textbf{(iii) Asymmetric linear loss.} For under- and over-prediction costs $c_u,c_o>0$,
consider $c_u(\Delta-c)_{+} +c_o(c-\Delta)_{+}$. If $\mathbb E[|\Delta|\mid H]<\infty$, zero is
Bayes-optimal if and only if the condition in (ii) holds with $\tau=c_u/(c_u+c_o)$.
\end{proposition}

Absolute loss corresponds to the median case $\tau=1/2$, yielding the following sign-based
characterization.

\begin{corollary}[Signed-change persistence under absolute loss]
\label{cor:absolute-main}
Let $\Delta$ be integrable conditional on $H$. Over $c\in\mathbb R$, the persistence correction $c=0$
minimizes $\mathbb E[|\Delta-c|\mid H]$ if and only if
\begin{equation}
  \sP(\Delta<0\mid H)\leq\tfrac12,
  \qquad
  \sP(\Delta>0\mid H)\leq\tfrac12.
  \label{eq:zero-median}
\end{equation}
It is the unique minimizer when zero is the
unique conditional median.
\end{corollary}

\noindent\emph{Proofs of Proposition~\ref{prop:common-losses-main} and Corollary~\ref{cor:absolute-main}
are provided in Appendix~\ref{app:loss-conditioned-decisions}.}

This characterization permits occurrence ranking and state movement to separate.

\begin{proposition}[Near-perfect occurrence ranking does not imply absolute-loss movement]
\label{prop:ranking-separation}
For every $\varepsilon>0$, there exist a binary-valued transition distribution and an occurrence-ranking score whose
AUROC exceeds $1-\varepsilon$, yet persistence is the unique conditional Bayes correction under absolute
loss at every context.
\end{proposition}

Even exact occurrence probabilities together with conditional variance need not determine whether
persistence is optimal.

\begin{theorem}[Non-identification of movability from occurrence and conditional variance]
\label{thm:summary-nonidentification}
Consider scalar corrections over $\mathcal D(H)=\mathbb R$. Fix any marginal law of $H$, measurable
function $p(H)\in(1/2,1)$, and measurable scale function $\alpha(H)\in(0,\infty)$. There exist two conditional transition
laws $P_0$ and $P_1$ such that, for $j\in\{0,1\}$,
\[
 \sP_j(\Delta\neq0\mid H)=p(H),
 \qquad
 \operatorname{Var}_j(\Delta\mid H)=p(H)\alpha(H)^2,
\]
and both induce the same joint law of $(H,Y)$ for $Y:=\mathbf 1\{\Delta\neq0\}$. Nevertheless, under
absolute loss, persistence is uniquely optimal under $P_0$, whereas a strictly positive correction is
uniquely optimal under $P_1$.
\end{theorem}

\noindent\emph{Proofs of Proposition~\ref{prop:ranking-separation} and Theorem~\ref{thm:summary-nonidentification}
are provided in Appendix~\ref{app:loss-conditioned-decisions}.}

The identical joint law of $(H,Y)$ implies that the two transition laws share occurrence probabilities,
Bernoulli entropies, and the receiver operating characteristic (ROC) curve of every $H$-measurable
occurrence score. Gates using only these occurrence summaries and conditional variance therefore have the same decision
distribution under both laws, including randomized gates with a common conditional seed law. In a numerical illustration
with a common fixed candidate, direct evaluation of its loss gain
over persistence distinguishes the two laws. Both the empirical-sign and lower confidence bound (LCB)
rules recover the opposite execution decisions as the calibration sample size increases
(see Appendix~\ref{app:summary-nonidentification-numerical}).

\subsection{Independent loss-specific calibration}

The separation above motivates evaluating a proposed correction directly through its downstream loss.
We use an independent calibration split to compare the fixed executed proposal with persistence within
predeclared groups.

\paragraph{Calibration setup.}
Let $\mathcal T$ denote the $\sigma$-field generated by the training data and all choices made before
calibration, including the executed proposal rule, group map, and any loss normalization or clipping. Conditioning on
$\mathcal T$ treats these quantities as fixed. Given $\mathcal T$, let $U_i$ denote a calibration unit
with target law $P_U$.
The unit sets the granularity at which independence is assumed: it may be a transition, series block,
episode, or complete training run. Multiple observations within a unit are summarized by a declared
bounded loss $L_B(\pi,U_i)\in[0,B]$, with $B>0$, for rule $\pi$. This may be the reported downstream loss when it is
naturally bounded, or a normalized and clipped surrogate specified before calibration. Each unit also has
pre-outcome grouping information $X_i$. A fixed map $g(X_i)\in\{1,\ldots,G\}$ assigns units to $G$
predeclared groups, and $\delta\in(0,1)$ denotes the desired simultaneous failure probability.

\paragraph{Groupwise gain.}
For persistence $\pi_0$ and the fixed executed proposal $\pi_Q$, define the observed gain on unit $i$ as
\begin{equation*}
  Z_i:=L_B(\pi_0,U_i)-L_B(\pi_Q,U_i)\in[-B,B].
\end{equation*}
A positive value favors the proposal. For each group with positive target-population probability, its
population-average gain is
\begin{equation*}
  \mu_g:=\mathbb E[Z_i\mid g(X_i)=g,\mathcal T].
\end{equation*}
In the single-transition case, $U_i=(H_i,\Delta_i)$ and $X_i=H_i$, with unit losses
$L_B(\pi_0,U_i)=\ell_B(0,\Delta_i)$ and
$L_B(\pi_Q,U_i)=\ell_B(c_Q(H_i),\Delta_i)$. If $P$ is the conditional transition law induced by $P_U$ given $\mathcal T$, then
\begin{equation*}
 \mu_g=\mathbb E\!\left[V_{\ell_B}(Q_\theta,P\mid H_i)\mid
 g(H_i)=g,\mathcal T\right].
\end{equation*}
This identity shows that $\mu_g$ averages pointwise proposal benefit over the histories in group $g$. For
general units, $\mu_g$ instead denotes the expected persistence-relative
gain under the declared bounded unit loss, including any aggregation or clipping.

\paragraph{Certified selection.}
For each group, let $I_g:=\{i:g(X_i)=g\}$ and $n_g:=|I_g|$. A group is \emph{represented} when
$n_g>0$. For such a group, define $\widehat\mu_g:=(\sum_{i\in I_g}Z_i)/n_g$ and
\begin{equation*}
  \LCB_g:=\widehat\mu_g-r_g,
  \qquad r_g:=B\sqrt{\frac{2\log(G/\delta)}{n_g}}.
\end{equation*}
Here $r_g$ is a one-sided Hoeffding radius, adjusted across the $G$ groups by a union bound
\citep{hoeffding1963probability}. In a represented group, the selective rule applies $\pi_Q$ when
$\LCB_g>0$ and applies persistence $\pi_0$ otherwise. Empty groups also receive persistence. In the
single-transition case, $c_{\mathrm{sel}}(H)=c_Q(H)$ when $g(H)$ is represented and
$\LCB_{g(H)}>0$, and $c_{\mathrm{sel}}(H)=0$ otherwise.
Theorem~\ref{thm:calibration-main} establishes the resulting simultaneous guarantee.

\begin{theorem}[Simultaneous accepted-group improvement]
\label{thm:calibration-main}
Conditionally on $\mathcal T$, suppose the proposal rule, groups, and bounded unit loss are fixed,
and the calibration units are i.i.d.\ from $P_U$. With conditional probability at least $1-\delta$,
every accepted group satisfies $\mu_g>0$, giving positive expected improvement over persistence
under $L_B$ and $P_U$.
\end{theorem}

\begin{proposition}[Finite-sample certification power]
\label{prop:certification-power-main}
Under the assumptions of Theorem~\ref{thm:calibration-main}, condition on $\mathcal T$ and the realized
group counts. With probability at least $1-\delta$, every represented group whose population gain satisfies
\begin{equation*}
 \mu_g>2B\sqrt{\frac{2\log(G/\delta)}{n_g}}
\end{equation*}
is accepted simultaneously.
\end{proposition}

\noindent\emph{Proofs of Theorem~\ref{thm:calibration-main} and Proposition~\ref{prop:certification-power-main}
are provided in Appendix~\ref{app:calibration-guarantees}.}

The sufficient per-group sample size scales as
$B^2\log(G/\delta)/\mu_g^2$, making explicit the dependence of simultaneous certification on the loss
bound, number of groups, and population gain.

\begin{samepage}
\section{Experiments}

The experiments address three questions:
\begin{list}{}{%
  \setlength{\leftmargin}{2em}%
  \setlength{\labelwidth}{1.6em}%
  \setlength{\labelsep}{0.4em}%
  \setlength{\itemsep}{3pt}%
  \setlength{\parsep}{0pt}%
}
  \item[Q1:] Can calibration recover a known loss-dependent movement boundary?
  \item[Q2:] How does proposal benefit vary across forecasting and dynamics settings?
  \item[Q3:] How does confidence correction affect update coverage and prediction loss?
\end{list}

Table~\ref{tab:evidence} summarizes each experiment's question, selection outcome, and evaluation result.
\end{samepage}

\begin{table}[H]
\centering
\small
\setlength{\tabcolsep}{9pt}
\renewcommand{\arraystretch}{1.08}
\caption{Experimental overview. Empirical losses are evaluated on held-out data. Lower loss and regret are better.}
\label{tab:evidence}
\begin{tabularx}{\textwidth}{@{}
  >{\raggedright\arraybackslash}p{0.17\textwidth}
  >{\raggedright\arraybackslash}p{0.27\textwidth}
  >{\raggedright\arraybackslash}X@{}}
\toprule
Scene & What it tests & Selection and held-out result \\
\midrule
Controlled synthetic changes & Q1, Q3: Movement boundary and finite-sample selection & $n_g: 50\!\to\!1000$. Power: $0.452\!\to\!0.888$. Regret: $0.0566\!\to\!0.0051$ \\
\addlinespace[1.5pt]
Monash Car Parts & Q2: Broad proposal benefit & Accept $2/2$ groups. Persistence / selective MAE: $0.573/0.391$ \\
\addlinespace[1.5pt]
M4 Monthly & Q2, Q3: Heterogeneous benefit and selective execution & Accept $1/3$ groups. Persistence / always execute / selective: $0.599/0.621/0.588$ \\
\addlinespace[1.5pt]
Minari FourRooms & Q2, Q3: Proposal benefit and selective execution in discrete-action dynamics & Accept $1/2$ groups. Persistence / always execute / selective: $0.185/0.094/0.122$ \\
\addlinespace[1.5pt]
Minari MuJoCo & Q2, Q3: Proposal benefit and selective execution across prediction horizons & At $h=20$, persistence / tuned uncertainty / selective NMSE: $1.919/0.552/0.552$. Selective is no worse on $8/9$ variants \\
\addlinespace[1.5pt]
Six-type unhealthy inventory & Q2, Q3: Fitted-proposal gain and held-out comparison & Candidate does not beat persistence at any $h$ \\
\bottomrule
\end{tabularx}
\end{table}
Throughout, $h$ denotes the forecast or rollout horizon. Update coverage denotes the frequency of
selecting the proposal. Proposal construction, gate calibration, and evaluation use disjoint data.
Proposals range from empirical and seasonal rules to learned ensembles and action-conditioned models.
Losses are compared within each setting. Appendix~\ref{app:assumption-audit} specifies calibration--deployment assumptions,
and Appendix~\ref{app:metric-definitions} defines the losses.

\subsection{Controlled phase transition}

A signed binary-change population varies
zero-change mass through the absolute-loss boundary at $1/2$.
Across 500 repetitions, increasing calibration units per group from 50 to 1,000 raises beneficial-group
power from 0.452 to 0.888 and lowers regret from 0.0566 to 0.0051.
Figure~\ref{fig:phase} shows the population boundary and its finite-sample recovery.
The fixed groups contain both beneficial and harmful proposals. We compare the LCB gate with an
empirical-sign rule that accepts any group with positive calibration mean gain. With 200 calibration
observations per group, exact binomial probabilities show that confidence correction reduces the probability of accepting
any harmful group from 6.96\% to 0.00010\%, while the acceptance rate for beneficial groups decreases
from 98.2\% to 66.6\%. Table~\ref{tab:phase-confidence} reports the resulting trade-off between
error control, update coverage, and population regret. A single transition law further illustrates loss
dependence: its Bayes correction is $0.4$ under squared loss, $0$ under absolute loss, and respectively
$-1$ and $3$ under pinball loss at $\tau=0.2$ and $0.8$. Table~\ref{tab:loss-conditioned-audit} reports
its exact risks.

\begin{figure}[!htbp]
  \centering
  \includegraphics[width=\linewidth]{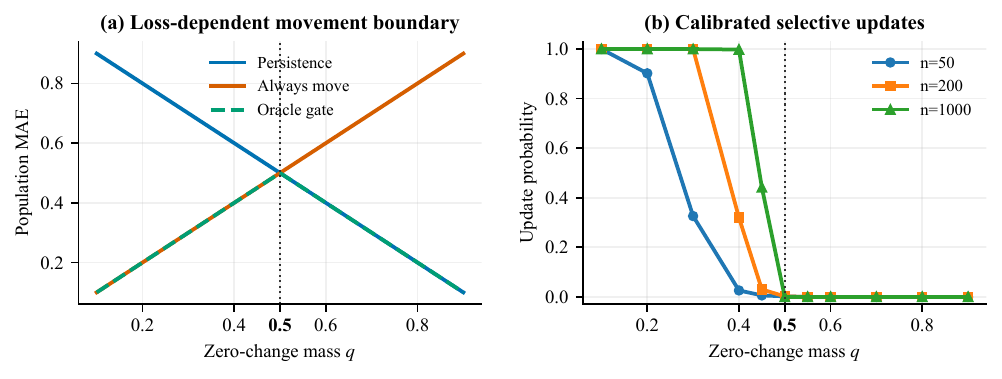}
  \caption{Loss-dependent movement phase transition. (a) Population MAE of persistence, always move, and the
  oracle gate. (b) Update probabilities for $n$ calibration observations per group. Dotted lines mark
  the population boundary at $q=0.5$.}
  \label{fig:phase}
\end{figure}

\subsection{Supported movement on intermittent demand}

On the public Monash Car Parts benchmark, comprising 2,674 monthly intermittent-demand
series \citep{godahewa2021monash}, a 27-month training block constructs groupwise empirical proposals,
followed by 12-month calibration and held-out test blocks. Calibration and testing use rolling
one-step forecasts with the latest observed history. A training zero-fraction threshold of $0.75$
defines dense and sparse groups. Both fitted absolute-loss Bayes actions predict zero demand.

Both groups are accepted, so selective execution equals always execute. Test MAE is $0.391$,
versus $0.573/0.626/0.539$ for persistence, seasonal naive, and a 12-month trailing mean, respectively.
Selective-minus-persistence is $-0.182$, with a 95\% bootstrap interval of $[-0.194,-0.170]$. MASE with lag-1 scaling
is $1.348$ versus $1.881$ over 2,504 eligible series. Finer post-hoc partitions
($K_{\mathrm{grp}}\in\{4,8\}$) refit the candidates and yield mixed test gains with zero gate coverage
(see Appendix~\ref{app:monash-sensitivity}).

\subsection{Selective execution under heterogeneous gain}

The M4 Monthly experiment
tests heterogeneous proposal gain under a protocol fixed before access to the official test outcomes
\citep{makridakis2018m4}. The 48,000 complete series are hash-split into 19,316 calibration and 28,684
held-out test series. The proposal repeats the latest 12-month cycle, and a training-only ratio of proposal to persistence
MAE defines three pre-specified groups. The simultaneous LCB accepts only the candidate-favored group,
covering 14.0\% of test series.

For absolute horizon error $e_j$ and training-history mean absolute first-difference scale $s$ floored at $10^{-8}$, the primary
series loss is $(\sum_{j=1}^{18}e_j/(e_j+s))/18\in[0,1]$. Always executing is worse than persistence
($0.621$ versus $0.599$), whereas selective execution reaches $0.588$. Selective-minus-persistence is
$-0.0114$ with paired 95\% bootstrap interval $[-0.0122,-0.0107]$. Selective-minus-always-execute is
$-0.0331$ with interval $[-0.0345,-0.0318]$.

The empirical-sign rule also accepts the ambiguous group, whose calibration mean gain is positive. Confidence correction restricts execution to the candidate-favored group, reducing coverage from
22.2\% to 14.0\% and certifying its gain under the stated sampling assumptions. Test loss increases
from $0.586$ to $0.588$, with a paired LCB-minus-sign difference of $0.00187$ and a 95\% bootstrap
interval of $[0.00146,0.00229]$ (see Table~\ref{tab:decision-ablation}).

The two extreme policies and the certified selector retain the same ordering under standard metrics:
persistence / always execute / selective obtain MASE with lag-1 scaling $3.380/4.259/3.301$ and symmetric mean absolute
percentage error (sMAPE) $0.153/0.160/0.147$. Appendix~\ref{app:certification-power} compares the observed group
gains and Hoeffding radii with the sufficient acceptance threshold in Proposition~\ref{prop:certification-power-main}.
Complementary Quarterly and Daily
evaluations respectively show insufficient Hoeffding certification power and nonpositive calibration gain
(see Appendix~\ref{app:cross-frequency-diagnostics}).

\subsection{Certified selection in discrete-action dynamics}

The Minari FourRooms
artifact contains 590 episodes and 10,010 transitions (\citeauthor{farama_fourrooms}). Ordered episodes are split
354/118/118 for fitting, calibration, and held-out testing. A train-only empirical table predicts the next
local image and direction from $(\text{image},\text{direction},\text{action})$. Unseen contexts persist. The
gate uses fixed turn/forward groups and a bounded loss equally weighting pixel mismatch and direction error.
Calibration and groupwise test evaluation average this loss within each episode and then equally across
episodes containing the group.

The simultaneous gain LCB is $0.029$ for the turn group and $-0.218$ for the forward group. Only the turn group is accepted, although both groups have positive calibration mean gains. On 117 held-out episodes containing turns, the accepted group's mean gain over
persistence is $0.2964$, with a paired 95\% bootstrap interval of $[0.2625,0.3292]$
(see Table~\ref{tab:fourrooms-group-gain}).

Across all 2,056 held-out transitions, transition-weighted action-conditioned / action-agnostic candidate losses are $0.094/0.125$, and
selective / persistence losses are $0.122/0.185$. The forward group also benefits on test, so selective execution has higher loss than always executing.
Selective-minus-persistence loss is $-0.0630$, with a paired episode-cluster 95\% interval of
$[-0.0692,-0.0568]$. Selective-minus-always-execute loss is $0.0277$, with an interval of
$[0.0249,0.0304]$. A closed-loop pilot records $2/30$ successful episodes for both selective and
persistence (see Appendix~\ref{app:fourrooms-closed-loop}).

\FloatBarrier
\subsection{Selective execution in neural multi-step dynamics}

The
neural dynamics study uses the Minari Hopper, HalfCheetah, and Walker2d dataset families
from the \citeauthor{farama_mujoco}. Across 27 protocol runs (three environments $\times$ three data levels $\times$ three seeds) and
$h\in\{1,5,10,20\}$, a three-member delta-MLP ensemble first generates an $h$-step open-loop prediction
under the recorded action sequence. The gate then selects between its terminal state prediction and the
window's initial state. Groups are horizon-specific training-only uncertainty tertiles. Calibration uses
window NMSE clipped to $[0,1]$, averaged within each episode--group pair and then equally across represented
episodes. Prediction accuracy is measured by unclipped NMSE, the scale-normalized coordinatewise squared
error averaged over windows within episodes, then over episodes, seeds, and dataset variants.

\begin{figure}[!htb]
  \centering
  \includegraphics[width=\linewidth]{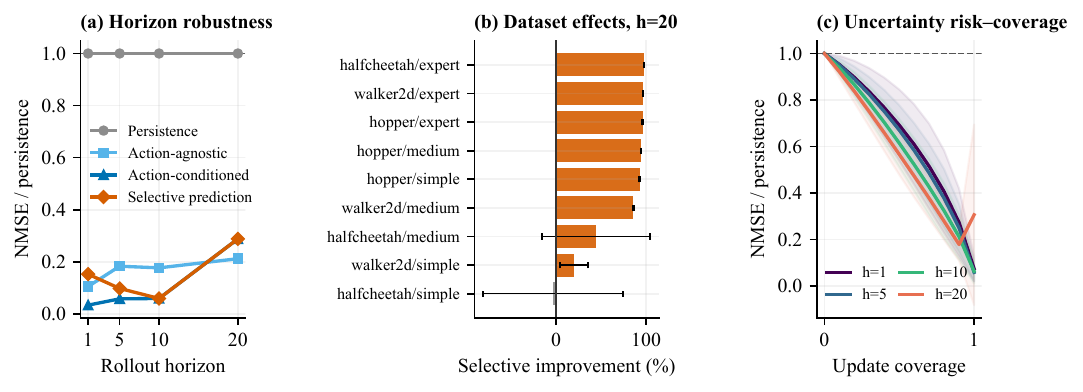}
  \caption{MuJoCo results. (a) NMSE relative to persistence across horizons. (b) Selective NMSE reduction (\%) relative to
  persistence at $h=20$, with means and sample standard deviations across seeds. (c) Relative NMSE
  versus uncertainty-based update coverage, with shaded bands showing one standard deviation across datasets.}
  \label{fig:mujoco}
\end{figure}

Selective prediction has lower macro NMSE than persistence at every horizon and is no worse on eight of nine
variants at horizon 20. For $h=1,5,10,20$, persistence-to-selective NMSE is respectively
$0.481\!\to\!0.073$, $0.793\!\to\!0.077$, $1.401\!\to\!0.082$, and $1.919\!\to\!0.552$, with coverage
$33.3/86.2/100/100\%$. A calibration-tuned uncertainty baseline accepts all groups and scores
$0.016/0.046/0.082/0.552$, improving on selective execution at $h=1,5$
(see Figure~\ref{fig:mujoco}).

Under the gate's clipped loss, all 324 run-specific groups have positive test gains in 27 repeated runs
with the same configuration and splits. Their calibration mean gains are also positive.
The LCB gate accepts 259 groups, including seven whose \emph{unclipped} test gains are negative
(see Table~\ref{tab:mujoco-bounded-gains}). It retains persistence in the remaining 65 groups, reducing
update coverage at shorter horizons, while the empirical-sign rule executes every proposal.
Under clipped loss, LCB-minus-sign differences are $0.0512$ and $0.0295$ at $h=1,5$, respectively
(see Table~\ref{tab:decision-ablation}).

\FloatBarrier
\subsection{Constrained forecasting of six unhealthy-inventory types}

We study 33 stock keeping units (SKUs) from JD.com in a supply-chain setting with six unhealthy-inventory types: shelf-life-risk,
aged, non-moving, slow-moving, off-shelf residual, and operationally non-saleable inventory. Observations span 56 days at the national SKU--day grain.
These quantities may overlap, and each is separately bounded by active inventory.
Appendix~\ref{app:inventory-theory} describes the overlapping state geometry and feasibility constraints.
For this research forecasting task, a shared probabilistic ensemble supplies a fixed, type-aware
state-update proposal, with feasibility enforced before calibration and evaluation.

\begin{figure}[!htb]
  \centering
  \includegraphics[width=\linewidth]{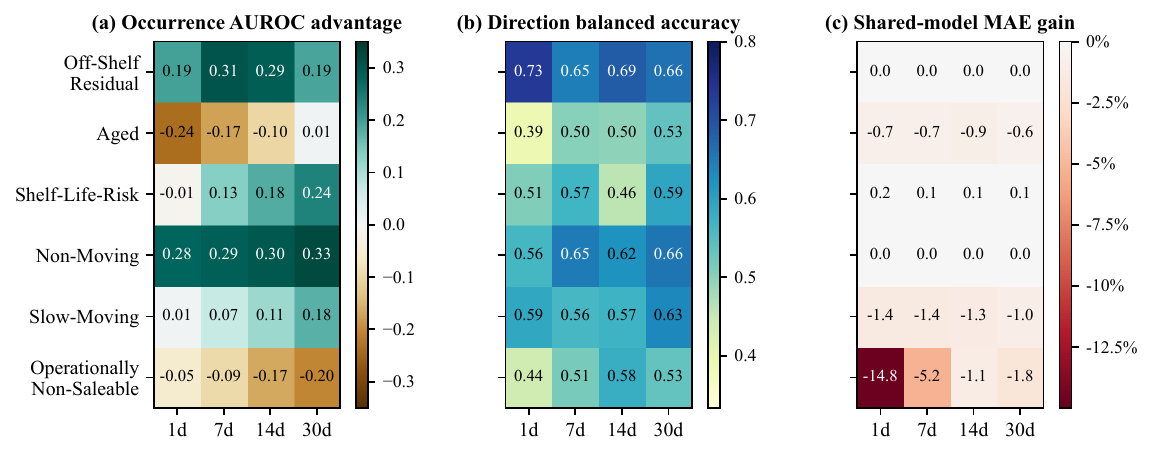}
  \caption{Diagnostics across six unhealthy-inventory types. (a) Occurrence-AUROC advantage over the strongest
  naive baseline. (b) Direction balanced accuracy. (c) Shared-model MAE gain (\%) relative to persistence.}
  \label{fig:inventory-map}
\end{figure}

The shared ensemble has higher MAE than persistence at each of the four horizons. At
$h\in\{1,7,14,30\}$, shared-model MAE is $1.030$, $3.967$, $6.850$, and $13.759$, compared with $1.027$,
$3.958$, $6.837$, and $13.737$ for persistence. This ordering is unchanged under inventory-normalized MAE
and weighted absolute percentage error (WAPE), across three training seeds, and in a later-window historical
backtest (see Appendix~\ref{app:inventory-robustness}). Yet occurrence probes attain macro AUROC $0.819$,
$0.792$, $0.791$, and $0.794$, whereas no type--horizon signed-delta probe achieves lower MAE than
zero-change persistence in at least four of the five entity-held-out folds. Thus, in this cohort, occurrence-ranking
signal coexists with a shared state-update proposal that does not improve MAE
(see Figure~\ref{fig:inventory-map}).

\FloatBarrier

\section{Conclusion}

A predictive transition is not yet a decision to change state. We formulate state execution as a
loss-conditioned choice between a fixed feasible proposal and persistence, with independent calibration certifying
when movement is supported. Across controlled, forecasting, and action-conditioned settings, the experiments
distinguish supported movement from predictive signals that do not justify a state update under the declared loss.
The result is a model-agnostic execution principle that applies across empirical, seasonal, action-conditioned,
and neural proposals.
Future work will study recursive rollouts that feed selected states into subsequent predictions,
with online recalibration and policy-aware proposal construction.

\clearpage
\setlength{\bibsep}{2pt plus 0.3ex}
\bibliography{references}

\begin{thebibliography}{31}
\providecommand{\natexlab}[1]{#1}
\providecommand{\url}[1]{\texttt{#1}}
\expandafter\ifx\csname urlstyle\endcsname\relax
  \providecommand{\doi}[1]{doi: #1}\else
  \providecommand{\doi}{doi: \begingroup \urlstyle{rm}\Url}\fi

\bibitem[Angelopoulos et~al.(2025)Angelopoulos, Bates, Cand\`es, Jordan, and
  Lei]{angelopoulos2025learntest}
Anastasios~N. Angelopoulos, Stephen Bates, Emmanuel~J. Cand\`es, Michael~I.
  Jordan, and Lihua Lei.
\newblock Learn then test: Calibrating predictive algorithms to achieve risk
  control.
\newblock \emph{The Annals of Applied Statistics}, 19\penalty0 (2):\penalty0
  1641--1662, 2025.
\newblock ISSN 1932-6157.
\newblock \doi{10.1214/24-AOAS1998}.
\newblock URL \url{https://doi.org/10.1214/24-AOAS1998}.

\bibitem[Bertsimas \& Kallus(2020)Bertsimas and
  Kallus]{bertsimas2020prescriptive}
Dimitris Bertsimas and Nathan Kallus.
\newblock From predictive to prescriptive analytics.
\newblock \emph{Management Science}, 66\penalty0 (3):\penalty0 1025--1044,
  2020.
\newblock ISSN 1526-5501.
\newblock \doi{10.1287/mnsc.2018.3253}.
\newblock URL \url{https://doi.org/10.1287/mnsc.2018.3253}.

\bibitem[Chen et~al.(2022)Chen, Sun, Li, Wang, Hao, and
  Sifakis]{chen2022simplexdrive}
Shengduo Chen, Yaowei Sun, Dachuan Li, Qiang Wang, Qi~Hao, and Joseph Sifakis.
\newblock Runtime safety assurance for learning-enabled control of autonomous
  driving vehicles.
\newblock In \emph{2022 International Conference on Robotics and Automation
  ({ICRA})}, pp.\  8978--8984. IEEE, 2022.
\newblock \doi{10.1109/ICRA46639.2022.9812177}.
\newblock URL \url{https://ieeexplore.ieee.org/document/9812177}.

\bibitem[Chua et~al.(2018)Chua, Calandra, McAllister, and Levine]{chua2018pets}
Kurtland Chua, Roberto Calandra, Rowan McAllister, and Sergey Levine.
\newblock Deep reinforcement learning in a handful of trials using
  probabilistic dynamics models.
\newblock In S.~Bengio, H.~Wallach, H.~Larochelle, K.~Grauman, N.~Cesa-Bianchi,
  and R.~Garnett (eds.), \emph{Advances in Neural Information Processing
  Systems}, volume~31. Curran Associates, Inc., 2018.
\newblock URL
  \url{https://proceedings.neurips.cc/paper/2018/hash/3de568f8597b94bda53149c7d7f5958c-Abstract.html}.

\bibitem[El-Yaniv \& Wiener(2010)El-Yaniv and Wiener]{elyaniv2010selective}
Ran El-Yaniv and Yair Wiener.
\newblock On the foundations of noise-free selective classification.
\newblock \emph{Journal of Machine Learning Research}, 11\penalty0
  (53):\penalty0 1605--1641, 2010.
\newblock URL \url{https://jmlr.org/papers/v11/el-yaniv10a.html}.

\bibitem[{Farama Foundation}({\natexlab{a}})]{farama_fourrooms}
{Farama Foundation}.
\newblock Fourrooms: Minari dataset documentation.
\newblock \url{https://minari.farama.org/main/datasets/minigrid/fourrooms/}.

\bibitem[{Farama Foundation}({\natexlab{b}})]{farama_mujoco}
{Farama Foundation}.
\newblock {MuJoCo}: Minari dataset documentation.
\newblock \url{https://minari.farama.org/main/datasets/mujoco/}.

\bibitem[Frauenknecht et~al.(2024)Frauenknecht, Eisele, Subhasish, Solowjow,
  and Trimpe]{frauenknecht2024macura}
Bernd Frauenknecht, Artur Eisele, Devdutt Subhasish, Friedrich Solowjow, and
  Sebastian Trimpe.
\newblock Trust the model where it trusts itself - model-based actor-critic
  with uncertainty-aware rollout adaption.
\newblock In \emph{Proceedings of the 41st International Conference on Machine
  Learning}, volume 235 of \emph{Proceedings of Machine Learning Research},
  pp.\  13973--14005. PMLR, 2024.
\newblock URL \url{https://proceedings.mlr.press/v235/frauenknecht24a.html}.

\bibitem[Geifman \& El-Yaniv(2019)Geifman and
  El-Yaniv]{geifman2019selectivenet}
Yonatan Geifman and Ran El-Yaniv.
\newblock {SelectiveNet}: A deep neural network with an integrated reject
  option.
\newblock In \emph{Proceedings of the 36th International Conference on Machine
  Learning}, volume~97 of \emph{Proceedings of Machine Learning Research}, pp.\
   2151--2159. PMLR, 2019.
\newblock URL \url{https://proceedings.mlr.press/v97/geifman19a.html}.

\bibitem[Giacomini \& White(2006)Giacomini and White]{giacomini2006conditional}
Raffaella Giacomini and Halbert White.
\newblock Tests of conditional predictive ability.
\newblock \emph{Econometrica}, 74\penalty0 (6):\penalty0 1545--1578, 2006.
\newblock \doi{10.1111/j.1468-0262.2006.00718.x}.
\newblock URL \url{https://doi.org/10.1111/j.1468-0262.2006.00718.x}.

\bibitem[Gneiting(2011)]{gneiting2011point}
Tilmann Gneiting.
\newblock Making and evaluating point forecasts.
\newblock \emph{Journal of the American Statistical Association}, 106\penalty0
  (494):\penalty0 746--762, 2011.
\newblock \doi{10.1198/jasa.2011.r10138}.
\newblock URL \url{https://doi.org/10.1198/jasa.2011.r10138}.

\bibitem[Gneiting \& Katzfuss(2014)Gneiting and
  Katzfuss]{gneiting2014forecasting}
Tilmann Gneiting and Matthias Katzfuss.
\newblock Probabilistic forecasting.
\newblock \emph{Annual Review of Statistics and Its Application}, 1\penalty0
  (1):\penalty0 125--151, 2014.
\newblock ISSN 2326-831X.
\newblock \doi{10.1146/annurev-statistics-062713-085831}.
\newblock URL \url{https://doi.org/10.1146/annurev-statistics-062713-085831}.

\bibitem[Gneiting \& Raftery(2007)Gneiting and Raftery]{gneiting2007scoring}
Tilmann Gneiting and Adrian~E. Raftery.
\newblock Strictly proper scoring rules, prediction, and estimation.
\newblock \emph{Journal of the American Statistical Association}, 102\penalty0
  (477):\penalty0 359--378, 2007.
\newblock ISSN 1537-274X.
\newblock \doi{10.1198/016214506000001437}.
\newblock URL \url{https://doi.org/10.1198/016214506000001437}.

\bibitem[Godahewa et~al.(2021)Godahewa, Bergmeir, Webb, Hyndman, and
  Montero-Manso]{godahewa2021monash}
Rakshitha~W Godahewa, Christoph Bergmeir, Geoffrey Webb, Rob Hyndman, and Pablo
  Montero-Manso.
\newblock Monash time series forecasting archive.
\newblock In J.~Vanschoren and S.~Yeung (eds.), \emph{Proceedings of the Neural
  Information Processing Systems Track on Datasets and Benchmarks}, volume~1,
  2021.
\newblock URL
  \url{https://datasets-benchmarks-proceedings.neurips.cc/paper_files/paper/2021/hash/eddea82ad2755b24c4e168c5fc2ebd40-Abstract-round2.html}.

\bibitem[Ha \& Schmidhuber(2018)Ha and Schmidhuber]{ha2018worldmodels}
David Ha and J\"{u}rgen Schmidhuber.
\newblock Recurrent world models facilitate policy evolution.
\newblock In S.~Bengio, H.~Wallach, H.~Larochelle, K.~Grauman, N.~Cesa-Bianchi,
  and R.~Garnett (eds.), \emph{Advances in Neural Information Processing
  Systems}, volume~31. Curran Associates, Inc., 2018.
\newblock URL
  \url{https://proceedings.neurips.cc/paper/2018/hash/2de5d16682c3c35007e4e92982f1a2ba-Abstract.html}.

\bibitem[Hand(2009)]{hand2009auc}
David~J. Hand.
\newblock Measuring classifier performance: A coherent alternative to the area
  under the {ROC} curve.
\newblock \emph{Machine Learning}, 77\penalty0 (1):\penalty0 103--123, 2009.
\newblock \doi{10.1007/s10994-009-5119-5}.
\newblock URL \url{https://doi.org/10.1007/s10994-009-5119-5}.

\bibitem[Hoeffding(1963)]{hoeffding1963probability}
Wassily Hoeffding.
\newblock Probability inequalities for sums of bounded random variables.
\newblock \emph{Journal of the American Statistical Association}, 58\penalty0
  (301):\penalty0 13--30, 1963.
\newblock \doi{10.1080/01621459.1963.10500830}.
\newblock URL \url{https://doi.org/10.1080/01621459.1963.10500830}.

\bibitem[Huang et~al.(2026)Huang, Ma, and Michailidis]{huang2026certificate}
Chenfeng Huang, Zixuan Ma, and George Michailidis.
\newblock Model-agnostic online certificate-driven calibration for time series
  forecasting under distribution shift.
\newblock In Emilija Perkovi\'c and Daniel Malinsky (eds.), \emph{Proceedings
  of the 42nd Conference on Uncertainty in Artificial Intelligence}, volume 337
  of \emph{Proceedings of Machine Learning Research}, pp.\  2244--2273. PMLR,
  2026.
\newblock URL \url{https://proceedings.mlr.press/v337/huang26b.html}.

\bibitem[In\'{a}cio et~al.(2026)In\'{a}cio, Cerqueira, Barandas, and
  Soares]{inacio2026selective}
Ricardo In\'{a}cio, Vitor Cerqueira, Mar\'{i}lia Barandas, and Carlos Soares.
\newblock Selective time series forecasting via metalearning.
\newblock arXiv preprint arXiv:2606.23448, 2026.
\newblock URL \url{https://arxiv.org/abs/2606.23448}.

\bibitem[Janner et~al.(2019)Janner, Fu, Zhang, and Levine]{janner2019mbpo}
Michael Janner, Justin Fu, Marvin Zhang, and Sergey Levine.
\newblock When to trust your model: Model-based policy optimization.
\newblock In H.~Wallach, H.~Larochelle, A.~Beygelzimer, Florence
  {d'Alch\'{e}-Buc}, E.~Fox, and R.~Garnett (eds.), \emph{Advances in Neural
  Information Processing Systems}, volume~32. Curran Associates, Inc., 2019.
\newblock URL
  \url{https://proceedings.neurips.cc/paper/2019/hash/5faf461eff3099671ad63c6f3f094f7f-Abstract.html}.

\bibitem[Joshi et~al.(2026)Joshi, Wang, Hassani, and
  Dobriban]{joshi2026riskcontrolled}
Sunay Joshi, Tao Wang, Hamed Hassani, and Edgar Dobriban.
\newblock Risk-controlled post-processing of decision policies.
\newblock arXiv preprint arXiv:2605.06479, 2026.
\newblock URL \url{https://arxiv.org/abs/2605.06479}.

\bibitem[Kidambi et~al.(2020)Kidambi, Rajeswaran, Netrapalli, and
  Joachims]{kidambi2020morel}
Rahul Kidambi, Aravind Rajeswaran, Praneeth Netrapalli, and Thorsten Joachims.
\newblock {MOReL}: Model-based offline reinforcement learning.
\newblock In H.~Larochelle, M.~Ranzato, R.~Hadsell, M.~F. Balcan, and H.~Lin
  (eds.), \emph{Advances in Neural Information Processing Systems}, volume~33,
  pp.\  21810--21823. Curran Associates, Inc., 2020.
\newblock URL
  \url{https://proceedings.neurips.cc/paper/2020/hash/f7efa4f864ae9b88d43527f4b14f750f-Abstract.html}.

\bibitem[Laroche et~al.(2019)Laroche, Trichelair, and Combes]{laroche2019spibb}
Romain Laroche, Paul Trichelair, and Remi Tachet~Des Combes.
\newblock Safe policy improvement with baseline bootstrapping.
\newblock In \emph{Proceedings of the 36th International Conference on Machine
  Learning}, volume~97 of \emph{Proceedings of Machine Learning Research}, pp.\
   3652--3661. PMLR, 2019.
\newblock URL \url{https://proceedings.mlr.press/v97/laroche19a.html}.

\bibitem[Lu et~al.(2026)Lu, Yao, He, Han, Liu, Liao, He, and
  Peng]{lu2026riseadaptive}
Hongbo Lu, Liang Yao, Chenghao He, Hao Han, Fan Liu, Wenlong Liao, Tao He, and
  Pai Peng.
\newblock {RISE}: Adaptive imagination for world action models.
\newblock arXiv preprint arXiv:2608.20430, 2026.
\newblock URL \url{https://arxiv.org/abs/2608.20430}.

\bibitem[Makridakis et~al.(2018)Makridakis, Spiliotis, and
  Assimakopoulos]{makridakis2018m4}
Spyros Makridakis, Evangelos Spiliotis, and Vassilios Assimakopoulos.
\newblock The {M4} competition: Results, findings, conclusion and way forward.
\newblock \emph{International Journal of Forecasting}, 34\penalty0
  (4):\penalty0 802--808, 2018.
\newblock \doi{10.1016/j.ijforecast.2018.06.001}.
\newblock URL \url{https://doi.org/10.1016/j.ijforecast.2018.06.001}.

\bibitem[Maurer \& Pontil(2009)Maurer and Pontil]{maurer2009empirical}
Andreas Maurer and Massimiliano Pontil.
\newblock Empirical bernstein bounds and sample-variance penalization.
\newblock In \emph{Proceedings of the 22nd Annual Conference on Learning
  Theory}, 2009.
\newblock URL \url{https://www.cs.mcgill.ca/~colt2009/papers/012.pdf}.

\bibitem[Mozannar \& Sontag(2020)Mozannar and Sontag]{mozannar2020defer}
Hussein Mozannar and David Sontag.
\newblock Consistent estimators for learning to defer to an expert.
\newblock In \emph{Proceedings of the 37th International Conference on Machine
  Learning}, volume 119 of \emph{Proceedings of Machine Learning Research},
  pp.\  7076--7087. PMLR, 2020.
\newblock URL \url{https://proceedings.mlr.press/v119/mozannar20b.html}.

\bibitem[Noskov et~al.(2024)Noskov, Fishkov, and Panov]{noskov2024selective}
Fedor Noskov, Alexander Fishkov, and Maxim Panov.
\newblock Selective nonparametric regression via testing.
\newblock In \emph{Proceedings of the 15th Asian Conference on Machine
  Learning}, volume 222 of \emph{Proceedings of Machine Learning Research},
  pp.\  1023--1038. PMLR, 2024.
\newblock URL \url{https://proceedings.mlr.press/v222/noskov24a.html}.

\bibitem[Tomar et~al.(2026)Tomar, Tirupathi, Daly, and
  Dusparic]{tomar2026shapelets}
Shivani Tomar, Seshu Tirupathi, Elizabeth Daly, and Ivana Dusparic.
\newblock Shapelets-enriched selective forecasting using time series foundation
  models.
\newblock arXiv preprint arXiv:2601.11821, 2026.
\newblock URL \url{https://arxiv.org/abs/2601.11821}.

\bibitem[Wang et~al.(2026)Wang, Zhang, Lin, Luo, Wang, Wang, and
  Qi]{wang2026trustimagination}
Rui Wang, Yue Zhang, Jiehong Lin, Kuncheng Luo, Jianan Wang, Zhongrui Wang, and
  Xiaojuan Qi.
\newblock When to trust imagination: Adaptive action execution for world action
  models.
\newblock arXiv preprint arXiv:2605.06222, 2026.
\newblock URL \url{https://arxiv.org/abs/2605.06222}.

\bibitem[Yu et~al.(2020)Yu, Thomas, Yu, Ermon, Zou, Levine, Finn, and
  Ma]{yu2020mopo}
Tianhe Yu, Garrett Thomas, Lantao Yu, Stefano Ermon, James Zou, Sergey Levine,
  Chelsea Finn, and Tengyu Ma.
\newblock {MOPO}: Model-based offline policy optimization.
\newblock In H.~Larochelle, M.~Ranzato, R.~Hadsell, M.~F. Balcan, and H.~Lin
  (eds.), \emph{Advances in Neural Information Processing Systems}, volume~33,
  pp.\  14129--14142. Curran Associates, Inc., 2020.
\newblock URL
  \url{https://proceedings.neurips.cc/paper/2020/hash/a322852ce0df73e204b7e67cbbef0d0a-Abstract.html}.

\end{thebibliography}
\bibliographystyle{iclr2027_conference}

\clearpage
\appendix

\section{Methodological Foundations}
\label{app:theory}

We position the execution interface among related methods, then establish loss-conditioned decisions,
proposal-quality bounds, and calibration guarantees.

\subsection{Comparison with Related Methods}
\label{app:method-comparison}

Related methods select predictions, policies, or rollout lengths using different evaluation targets
(see Table~\ref{tab:method-comparison}). Our interface selects between a feasible state correction and
persistence using the correction's group-average loss reduction. Evaluating the proposal after the
feasibility map aligns the calibration target with the state update actually executed.

The statistical construction follows the learn-then-test principle of calibrating fixed candidates through
multiple testing \citep{angelopoulos2025learntest}: each group tests the null $\mu_g\leq0$, and simultaneous
error control certifies the accepted groups. Hoeffding bounds instantiate this construction for bounded
loss gains under conditional i.i.d.\ sampling (see Appendix~\ref{app:assumption-audit}). Other valid simultaneous
bounds can be used with the same feasible proposal, loss, and persistence baseline.

\begin{table}[ht]
\centering
\small
\setlength{\tabcolsep}{3.2pt}
\caption{Decision objects and evaluation targets of representative selection methods.}
\label{tab:method-comparison}
\begin{tabularx}{\textwidth}{@{}
>{\raggedright\arraybackslash}p{0.30\textwidth}
>{\raggedright\arraybackslash}p{0.18\textwidth}
>{\raggedright\arraybackslash}p{0.20\textwidth}
>{\raggedright\arraybackslash}X@{}}
\toprule
Method & Decision object & Retention or fallback & Evaluation target \\
\midrule
Selective prediction \citep{elyaniv2010selective} & Predict or abstain & Abstention & Risk on accepted predictions and coverage \\
Conditional forecast selection \citep{giacomini2006conditional} & Forecast rule & Alternative forecast & Conditional expected loss difference \\
Learn-then-test \citep{angelopoulos2025learntest} & Candidate configuration & Application-defined & User-specified risk constraints \\
Online forecast calibration \citep{huang2026certificate} & Gated residual correction & Backbone forecast & Target-risk certificate under temporal dependence and shift \\
Baseline bootstrapping \citep{laroche2019spibb} & Policy improvement & Baseline policy in uncertain regions & Return relative to the baseline \\
Risk-controlled post-processing \citep{joshi2026riskcontrolled} & Policy modification & Retain baseline where possible & Baseline agreement subject to a chance-risk constraint \\
Adaptive imagination \citep{lu2026riseadaptive} & Imagination depth & Stop imagination & Expected planning benefit versus computation cost \\
Adaptive action execution \citep{wang2026trustimagination} & Action-chunk length & Replan & Prediction--observation consistency \\
Loss-conditioned execution (ours) & Feasible state correction & Persist current state & Positive group-average bounded-loss gain over persistence \\
\bottomrule
\end{tabularx}
\end{table}
\FloatBarrier

\subsection{Loss-Conditioned Decisions}
\label{app:loss-conditioned-decisions}

Let $H$ be a forecast context, including an action when one is observed, and let
$\Delta\in\mathbb R^d$ be the future state correction, where $d$ is the state dimension. A fitted transition model supplies
the model-induced conditional predictive law $Q(\Delta\mid H)$ for this correction. Let $S$ be the current feasible state,
$\mathcal A(H)$ a raw correction domain,
$\mathcal D(H)$ the set of feasible corrections containing zero, and $F_H$ a fixed map into
$S+\mathcal D(H)$. Use measurable rules and finite conditional risks and risk infima, and assume the
Bayes minimum below is attained. For a loss $\ell(c,\Delta)$, define
\begin{align*}
 R_Q(c\mid H)&:=\mathbb E_Q[\ell(c,\Delta)\mid H],\\
 b_Q(H)&\in\arg\min_{c\in\mathcal A(H)}R_Q(c\mid H),\\
 c_Q(H)&:=F_H(S+b_Q(H))-S,\\
 R_P(c\mid H)&:=\mathbb E_P[\ell(c,\Delta)\mid H],\\
 M_\ell(P\mid H)&:=R_P(0\mid H)-\inf_{c\in\mathcal D(H)}R_P(c\mid H),\\
 V_\ell(Q,P\mid H)&:=R_P(0\mid H)-R_P(c_Q(H)\mid H).
\end{align*}
We suppress the parameter subscript $\theta$ in this appendix, so $Q$ denotes the fitted predictive law
$Q_\theta$ used in the main text. The quantity $c_Q(H)$ is the feasible correction actually executed from the
model's proposal. All calibration losses are evaluated after applying $F_H$. Direct constrained Bayes
optimization is recovered by setting $\mathcal A(H)=\mathcal D(H)$ and taking $F_H$ to be the identity on feasible states.
Persistence corresponds to the feasible correction $c_0(H):=0$. In this notation, distributional informativeness
is a property of $Q$, population state movability is captured by $M_\ell>0$, and proposal benefit is measured
by $V_\ell$. Positive $V_\ell$ gives population support for the proposal. Calibration assesses evidence for
positive group-average gain under the declared bounded unit loss.

\begin{proof}[Proof of Proposition~\ref{prop:common-losses-main} and
Corollary~\ref{cor:absolute-main}]
For squared loss,
\[
  \mathbb E[(\Delta-c)^2\mid H]
  =\mathbb E[\Delta^2\mid H]-2c\mathbb E[\Delta\mid H]+c^2.
\]
The strictly convex quadratic has the unique minimizer $c=\mathbb E[\Delta\mid H]$, proving (i). For (ii), let $G_H(x):=\sP(\Delta\leq x\mid H)$. The left and right derivatives of the convex pinball
risk at $c$ are $G_H(c^-)-\tau$ and $G_H(c)-\tau$. Thus its minimizers are exactly the corrections
satisfying $G_H(c^-)\leq\tau\leq G_H(c)$. Substituting $c=0$ gives Equation~(\ref{eq:zero-quantile}).
Dividing the asymmetric linear loss by $c_u+c_o>0$ preserves its minimizers and gives pinball loss
with $\tau=c_u/(c_u+c_o)$, proving (iii). Absolute loss is twice pinball loss at $\tau=1/2$, so its minimizers are exactly the conditional
medians. Substituting $c=0$ yields Equation~(\ref{eq:zero-median}), and uniqueness holds exactly when
zero is the only conditional median.
\end{proof}

\begin{corollary}[Strict persistence margin from a zero-change atom]
\label{cor:zero-mass-margin}
Let $\Delta$ be scalar and integrable conditional on $H$, with $q:=\sP(\Delta=0\mid H)>1/2$. Under absolute loss, persistence is
the unique Bayes correction and every $c\neq0$ satisfies
\begin{equation*}
 \mathbb E[|\Delta-c|-|\Delta|\mid H]\geq(2q-1)|c|>0.
\end{equation*}
\end{corollary}

\begin{proof}
Fix $H$ and let $c\neq0$. On the event $\{\Delta=0\}$, the excess absolute loss is exactly $|c|$.
On $\{\Delta\neq0\}$, the reverse triangle inequality yields
$|\Delta-c|-|\Delta|\geq-|c|$. Taking conditional expectations and using
$\sP(\Delta=0\mid H)=q$ gives
\[
 \mathbb E[|\Delta-c|-|\Delta|\mid H]
 \geq q|c|-(1-q)|c|=(2q-1)|c|.
\]
The assumption $q>1/2$ makes this lower bound strictly positive for every $c\neq0$. Hence, $c=0$ is the unique
Bayes correction under absolute loss.
\end{proof}

\begin{table}[ht]
\centering
\small
\caption{Exact loss-conditioned decisions for the same law
$\sP(\Delta=-1,0,3)=(0.35,0.40,0.25)$. Bold marks the lower value, and ties are both highlighted.}
\label{tab:loss-conditioned-audit}
\begin{tabular}{lS[table-format=1.1]S[table-format=-1.1]S[table-format=1.2]S[table-format=1.2]}
\toprule
Loss & \multicolumn{1}{c}{Level} & \multicolumn{1}{c}{Bayes correction} & \multicolumn{1}{c}{Persistence risk} & \multicolumn{1}{c}{Bayes risk} \\
\midrule
Squared & \multicolumn{1}{c}{---} & 0.4 & 2.60 & \bfseries 2.44 \\
Absolute & \multicolumn{1}{c}{---} & 0 & \bfseries 1.10 & \bfseries 1.10 \\
Pinball & 0.2 & -1 & 0.43 & \bfseries 0.28 \\
Pinball & 0.5 & 0 & \bfseries 0.55 & \bfseries 0.55 \\
Pinball & 0.8 & 3 & 0.67 & \bfseries 0.52 \\
\bottomrule
\end{tabular}
\end{table}

\begin{proof}[Proof of Proposition~\ref{prop:ranking-separation}]
Fix $r\in(0,1/2)$. Let $H\in\{0,1\}$ with $\sP(H=1)=w$, where $w\in(0,1)$ will be chosen below, and
let $\Delta=Y\in\{0,1\}$ satisfy $\sP(Y=1\mid H=0)=0$ and $\sP(Y=1\mid H=1)=r$. Consider the
occurrence-ranking score $s(H)=H$. Every positive example has score one. Conditional on $Y=0$, the score
equals zero with probability $(1-w)/(1-wr)$ and one otherwise. Under the standard half-credit convention
for ties,
\[
 \operatorname{AUROC}(s)
 =\frac{1-w}{1-wr}+\frac{w(1-r)}{2(1-wr)}
 =1-\frac{w(1-r)}{2(1-wr)}.
\]
This expression converges to one as $w\downarrow0$, so choosing $w>0$ sufficiently small makes the
AUROC exceed $1-\varepsilon$. In both contexts, $\sP(\Delta=1\mid H)<1/2$. Hence, zero is the unique
conditional median, and therefore the unique Bayes correction under absolute loss, by
Corollary~\ref{cor:absolute-main}. Thus, occurrence can be ranked arbitrarily well even when persistence
is optimal at every context.
\end{proof}

\begin{proof}[Proof of Theorem~\ref{thm:summary-nonidentification}]
Use the following two scalar conditional transition laws. Under $P_0$,
\begin{equation*}
 \sP_0(\Delta=-\alpha(H)\mid H)=\frac{p(H)}{2},\quad
 \sP_0(\Delta=0\mid H)=1-p(H),\quad
 \sP_0(\Delta=\alpha(H)\mid H)=\frac{p(H)}{2}.
\end{equation*}
Under $P_1$, with $\beta(H):=\alpha(H)/\sqrt{1-p(H)}$,
\begin{equation*}
 \sP_1(\Delta=0\mid H)=1-p(H),\qquad
 \sP_1(\Delta=\beta(H)\mid H)=p(H).
\end{equation*}
Write $p:=p(H)$, $\alpha:=\alpha(H)$, and $\beta:=\beta(H)$ while conditioning on $H$. Under both laws,
$\mathbf 1\{\Delta\neq0\}$ is Bernoulli with parameter $p$. Together with the fixed marginal law of $H$,
this establishes that the complete joint law of the context and occurrence label is identical. A score $s(H)$ therefore induces
the same score--label law, and hence the same ROC curve and AUROC, under $P_0$ and $P_1$. The occurrence
probability and its Bernoulli entropy are functions of the same $p$.

Under $P_0$, symmetry gives $\mathbb E_0[\Delta\mid H]=0$ and
$\operatorname{Var}_0(\Delta\mid H)=p\alpha^2$.
Under $P_1$, $\mathbb E_1[\Delta\mid H]=p\beta$ and
\begin{equation*}
 \operatorname{Var}_1(\Delta\mid H)
 =p\beta^2-p^2\beta^2=p(1-p)\beta^2=p\alpha^2,
\end{equation*}
where the last equality uses $\beta=\alpha/\sqrt{1-p}$.

Finally, under $P_0$, the conditional CDF jumps from $p/2<1/2$ to $1-p/2>1/2$ at zero.
Thus zero is the unique conditional median by Corollary~\ref{cor:absolute-main}. Under $P_1$, the point $\beta$ has mass $p>1/2$, while
the only remaining mass lies at zero, making $\beta$ the unique conditional median. Absolute-loss Bayes
corrections therefore differ despite equality of every summary named in the theorem. Any gate based solely
on these summaries, including a randomized gate with a common conditional seed law, has the same decision
distribution under $P_0$ and $P_1$ and therefore cannot recover both execution decisions almost surely.
\end{proof}

\subsection{Calibration Guarantees}
\label{app:calibration-guarantees}

Let $B>0$ denote the unit-loss bound, $G\geq1$ the number of groups,
$\delta\in(0,1)$ the target failure probability, and $\mathcal T$ the
$\sigma$-field generated by training and all design choices made before calibration. Conditional on
$\mathcal T$, the executed proposal rule $\pi_Q$, persistence
rule $\pi_0$, pre-outcome group map $g$, clipping rule, and number of groups are fixed. Each calibration unit
$U_i$ contains pre-outcome grouping information $X_i$. The $n$ calibration units $U_1,\ldots,U_n$ are
conditionally i.i.d.\ from a target law $P_U$, which is also the population covered by the guarantee. A unit
may be a transition, series block, episode, or complete training run. Dependence within a unit is allowed. The
units themselves are the independent sampling objects. The declared unit loss satisfies
$L_B(\pi,U)\in[0,B]$, and empty groups receive persistence.
Appendix~\ref{app:assumption-audit} relates these assumptions to the experimental populations.

\begin{proof}[Proof of Theorem~\ref{thm:calibration-main}]
Condition on $\mathcal T$ and the complete vector of calibration group labels. For a represented group,
conditional i.i.d.\ sampling implies that its $Z_i$ are independent, share conditional mean $\mu_g$, and lie
in an interval of width $2B$. For any $t>0$, one-sided Hoeffding concentration \citep{hoeffding1963probability} gives
\[
 \sP\!\left(\mu_g<\widehat\mu_g-t\mid
 \mathcal T,(g(X_j))_{j=1}^n\right)
 \leq\exp\!\left(-\frac{n_g t^2}{2B^2}\right).
\]
Taking $t=B\sqrt{2\log(G/\delta)/n_g}$ makes the right-hand side $\delta/G$. A union bound over the
$G$ predeclared groups makes every represented lower bound valid simultaneously with probability at least
$1-\delta$. An accepted group has a positive lower bound and hence $\mu_g>0$. Acceptance uses the executed
proposal rule, whose expected unit-loss advantage is $\mu_g$, whereas rejection executes persistence itself.
Averaging over the group-label vector proves the claim conditional on $\mathcal T$.
Averaging over $\mathcal T$ also gives the unconditional statement.
\end{proof}

\begin{proof}[Proof of Proposition~\ref{prop:certification-power-main}]
Theorem~\ref{thm:calibration-main} controls upward deviations of the sample mean, whereas this power statement uses the
opposite tail. For each represented group, Hoeffding's inequality gives
\[
 \sP\!\left(\widehat\mu_g<\mu_g-r_g\mid
 \mathcal T,(g(X_j))_{j=1}^n\right)
 \leq\exp\!\left(-\frac{n_g r_g^2}{2B^2}\right)=\frac{\delta}{G},
\]
where $r_g$ is the Hoeffding radius defined in the main text. A union bound yields
$\widehat\mu_g\geq\mu_g-r_g$ simultaneously with probability at least $1-\delta$.
On this event, each group with $\mu_g>2r_g$ satisfies
$\LCB_g=\widehat\mu_g-r_g\geq\mu_g-2r_g>0$ and is accepted.
Averaging over label vectors with the same counts gives the stated conditional guarantee.
For $\mu_g>0$, rearranging the threshold yields $n_g>8B^2\log(G/\delta)/\mu_g^2$.
\end{proof}

\paragraph{Episodes contributing to multiple groups.}
\label{app:episode-group-calibration}
In FourRooms and MuJoCo, one episode can contain observations from several groups.
Let $U_i$ be a complete episode, with episodes conditionally i.i.d.\ given $\mathcal T$.
For each fixed group $g$, define $A_g(U_i)\in\{0,1\}$ to indicate whether the episode contains that
group, and let $L_{B,g}(\pi,U_i)\in[0,B]$ be its within-group mean loss when $A_g(U_i)=1$.
For groups with $\sP(A_g(U)=1\mid\mathcal T)>0$, the population gain is
\[
 \mu_g^{\mathrm{ep}}
 =\mathbb E\!\left[L_{B,g}(\pi_0,U)-L_{B,g}(\pi_Q,U)
 \mid A_g(U)=1,\mathcal T\right].
\]
The proposal, membership rule, and loss aggregation are fixed by $\mathcal T$.
For a given $g$, conditioning on $(A_g(U_i))_i$ leaves independent represented-episode gains with
mean $\mu_g^{\mathrm{ep}}$ and range $[-B,B]$. For $n_g=\sum_i A_g(U_i)>0$, the same Hoeffding radius
therefore applies. Empty groups receive persistence. Integrating over these indicators gives failure probability at most
$\delta/G$ for that group. A union bound yields simultaneous accepted-group improvement.
Groups may share episodes because the union bound permits dependence between groupwise tests.
This certificate concerns the mean gain among episodes containing the group, using the aggregation
defined in Appendix~\ref{app:metric-definitions}.

\section{Six-Type Inventory Model}
\label{app:inventory-theory}

This section describes the abstract state constraints for the six unhealthy-inventory types used in the research forecasting model. For one national SKU-day, let $A\geq0$ denote active inventory. The six unhealthy-inventory coordinates form an overlapping marginal vector
$u\in\mathbb R_+^K$, $K:=6$, with feasible box
$\mathcal B_K(A):=\{u:0\leq u^{(k)}\leq A\ \text{for every }k\}$, where $k$ indexes the six types. No constraint is imposed on
$\sum_k u^{(k)}$, because a physical unit can satisfy multiple operational rules.
Figure~\ref{fig:inventory-overlap-geometry} illustrates this overlap in a three-type projection.
Proposition~\ref{prop:overlap} quantifies the irreducible error incurred by replacing this box with a disjoint simplex. Define the disjoint alternative $\mathcal S_K(A):=\{v\in\mathbb R_+^K:\sum_k v^{(k)}\leq A\}$ and overlap excess $\Omega(u,A):=(\sum_k u^{(k)}-A)_{+}$.

\begin{figure}[!htb]
\centering
\includegraphics[width=0.75\linewidth]{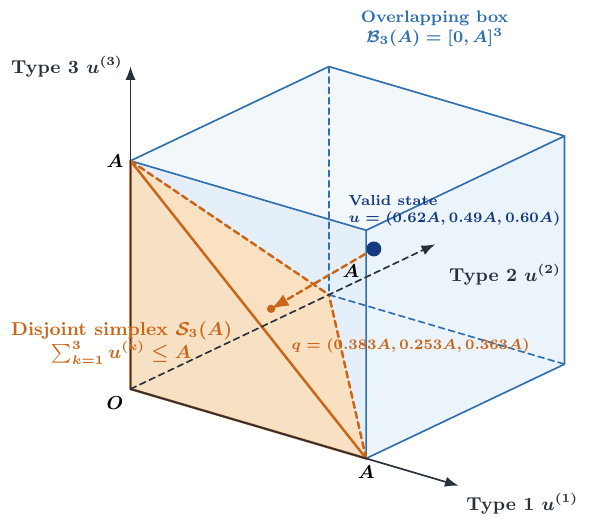}
\caption{Overlapping-inventory geometry. In a three-type projection, the disjoint simplex (orange)
excludes a valid state in the feasible box (blue). The dashed segment is its Euclidean distance to the simplex.}
\label{fig:inventory-overlap-geometry}
\end{figure}

\begin{proposition}[Irreducible error under a disjoint representation]
\label{prop:overlap}
For every $u\in\mathcal B_K(A)$ and $v\in\mathcal S_K(A)$,
\begin{equation*}
 \|u-v\|_1\geq \Omega(u,A),
 \qquad
 \|u-v\|_2\geq\frac{\Omega(u,A)}{\sqrt K}.
\end{equation*}
Thus positive overlap excess creates representation error before any forecasting error for a
simplex-constrained model.
\end{proposition}

\begin{proof}
If $\Omega(u,A)=0$, the lower bounds are immediate.
Otherwise, for any $v\in\mathcal S_K(A)$,
\[
 \|u-v\|_1
 \geq\left|\sum_k(u^{(k)}-v^{(k)})\right|
 \geq\sum_k u^{(k)}-A=\Omega(u,A).
\]
\end{proof}

The research model applies this feasibility map to candidate states before computing calibration
or evaluation losses. Each of the six type quantities is constrained independently to available active
inventory, preserving the overlapping representation.

\section{Experimental Details and Additional Results}
\label{app:additional-diagnostics}

This appendix specifies the evaluation protocol and provides additional diagnostics, robustness checks,
and benchmark results.

\subsection{Evaluation Protocol}

We define the evaluation losses, calibration assumptions, and statistical intervals used to interpret
the reported results.

\subsubsection{Evaluation Losses}
\label{app:metric-definitions}

Calibration compares persistence and the candidate using the same unit loss $L_B\in[0,1]$.
The definitions below specify normalization, clipping, and aggregation for each empirical setting.
Candidate construction and calibration have distinct loss roles: a Bayes proposal targets its declared
construction loss, while the gate evaluates the fixed executed proposal under $L_B$.
Monash, for example, fits empirical medians under absolute loss and calibrates their clipped normalized
absolute error. M4 and Web Traffic use fixed seasonal proposals. The calibration guarantee applies to
each fixed proposal through its observed bounded-loss gains, including when its construction loss differs
from $L_B$. FourRooms and MuJoCo use the episode--group formulation in
Appendix~\ref{app:episode-group-calibration}.

\paragraph{Six-type unhealthy inventory.}
For SKU $i$, let $\mathcal O_{i,h}$ be its valid forecast origins at horizon $h$, and define
$e_{ioj}^{(k)}:=|\widehat u_{i,o+j}^{(k)}-u_{i,o+j}^{(k)}|$ for origin $o$, lead $j$, and type $k$.
With observed target active inventory $A_{i,o+j}$, the calibration loss is
\[
 L^{\mathrm{inv}}_{B,i,h}
 =\frac{1}{6h|\mathcal O_{i,h}|}
 \sum_{o\in\mathcal O_{i,h}}\sum_{j=1}^{h}\sum_{k=1}^{6}
 \min\!\left\{\frac{e_{ioj}^{(k)}}{\max(A_{i,o+j},1)},1\right\}.
\]
Each type error is normalized and clipped before averaging over types, leads, and origins.
Reported MAE averages the raw errors in the same order, then gives each SKU equal weight.
Inventory-normalized MAE omits clipping. Inventory WAPE pools raw errors and target active inventory
across SKUs, origins, and leads separately for each type, then averages the six ratios.
The pooled inventory denominator is floored at one.

\paragraph{Monash Car Parts.}
Let $e_{ij}$ be the absolute error at evaluation month $j$ for series $i$, and let
$s_i$ be its mean absolute first difference over the initial 27 training months. For each 12-month block,
\[
 L^{\mathrm{Monash}}_{B,i}
 =\frac1{12}\sum_{j=1}^{12}\min\!\left\{\frac{e_{ij}}{\max(s_i,1)},1\right\}.
\]
Test MAE averages raw errors across series and months. MASE averages each series' MAE divided by
the unfloored $s_i$ over the 2,504 series with $s_i>0$.

\paragraph{M4 and Web Traffic Weekly.}
For absolute forecast error $e_{ij}$, define $s_i$ as the mean absolute first difference of the
history available at the forecast origin, floored at $10^{-8}$. The bounded series loss is
\[
 L^{\mathrm{forecast}}_{B,i}=\frac1h\sum_{j=1}^{h}\frac{e_{ij}}{e_{ij}+s_i}.
\]
The horizons are $h=18,8,14$ for M4 Monthly, Quarterly, and Daily, respectively, and $h=8$ for
Web Traffic Weekly. Reported metrics give each series equal weight. MASE divides series MAE by $s_i$.
sMAPE averages $2e_{ij}/(|y_{ij}|+|\widehat y_{ij}|)$, with a zero contribution when both values are zero.
sMAPE is reported as a ratio on $[0,2]$.

\paragraph{Minari FourRooms.}
For a predicted image $\widehat x$, target image $x$ with $m$ scalar entries, and predicted and target
directions $\widehat d,d$, the transition loss is
\[
 \ell^{\mathrm{FR}}=\frac12\left(
 \frac1m\sum_{r=1}^{m}\mathbf1\{\widehat x_r\neq x_r\}
 +\mathbf1\{\widehat d\neq d\}\right).
\]
For episode $i$ and action group $g$, let $\mathcal J_{ig}$ contain its transitions in that group.
Each represented episode contributes the group mean
$L^{\mathrm{FR}}_{B,ig}=(\sum_{j\in\mathcal J_{ig}}\ell^{\mathrm{FR}}_{ij})/|\mathcal J_{ig}|$
to calibration and groupwise test evaluation. Table~\ref{tab:fourrooms-group-gain} reports the resulting
episode-average test gains. Overall prediction loss weights these means by $|\mathcal J_{ig}|$,
giving every transition equal weight.
Bootstrap procedures are specified in Section~\ref{app:statistical-units}.

\paragraph{Minari MuJoCo.}
For rollout window $w$ in episode $i$, let $\widehat s_{iw}$ and $s_{iw}$ be the predicted and target
states at the specified horizon. With training-state coordinate standard deviations
$\sigma_k$ floored at $10^{-6}$, the window NMSE is
\[
 q_{iw}=\frac1d\sum_{k=1}^{d}
 \left(\frac{\widehat s_{iw,k}-s_{iw,k}}{\sigma_k}\right)^2.
\]
Let $\mathcal W_{ig}$ contain the episode's windows in uncertainty group $g$ at that horizon.
Calibration and groupwise bounded-loss evaluation clip each window NMSE before taking the represented
episode--group mean:
\[
 L^{\mathrm{MJ}}_{B,ig}=\frac1{|\mathcal W_{ig}|}
 \sum_{w\in\mathcal W_{ig}}\min\{q_{iw},1\}.
\]
Reported NMSE averages unclipped $q_{iw}$ over all evaluated windows within each episode, then equally
over episodes. Cross-dataset results give each of the nine variants equal weight after averaging its
three seeds. Thus $L^{\mathrm{MJ}}_{B,ig}$ is the groupwise certification target, and NMSE summarizes
prediction error over each complete episode.

\subsubsection{Calibration and Deployment Assumptions}
\label{app:assumption-audit}

Theorem~\ref{thm:calibration-main} certifies positive group-average gain over persistence for calibration
units that are i.i.d.\ conditional on the fixed training and design choices $\mathcal T$. The guarantee
uses the declared bounded unit loss, including its aggregation and clipping
(see Appendix~\ref{app:metric-definitions}). Transfer to deployment requires positive gain to persist in
accepted groups. A common calibration and deployment population is a sufficient condition.
Table~\ref{tab:assumption-audit} records the sampling units, split designs, and population-alignment
conditions for each experiment.

\begin{table}[ht]
\centering
\small
\setlength{\tabcolsep}{3.2pt}
\caption{Calibration units and split designs. Population alignment is assessed under the theorem's
conditional-i.i.d.\ sampling assumption.}
\label{tab:assumption-audit}
\begin{tabularx}{\textwidth}{@{}
>{\raggedright\arraybackslash}p{0.13\textwidth}
>{\raggedright\arraybackslash}p{0.18\textwidth}
>{\raggedright\arraybackslash}p{0.27\textwidth}
>{\raggedright\arraybackslash}X@{}}
\toprule
Setting & Calibration unit & Split design & Population alignment \\
\midrule
Controlled & Generated observation within a group & Independent draws with a fixed count per group & Same generator. Groupwise concentration applies directly \\
Monash & One calibration block per series & Same series, consecutive 12-month blocks & Group-gain stability across forecast origins \\
M4 frequencies & Complete series & Hash-disjoint series. Calibration targets precede the official origin & Common series population and temporal group-gain stability \\
Web Traffic Weekly & Complete page series & Hash-disjoint pages with the final eight observations per page & Common page population. Pooled-history grouping evaluated empirically \\
FourRooms & Episode & Ordered episode prefixes & Stable episode collection law across the split \\
MuJoCo & Episode & Fixed-seed random episode permutation & Common episode population within each dataset \\
Unhealthy inventory & Complete national SKU trajectory & Entity-disjoint folds in one window & Common SKU population. Later-window evaluation assumes temporal gain stability \\
\bottomrule
\end{tabularx}
\end{table}
\FloatBarrier

Each of the 500 controlled repetitions draws a fresh calibration sample within every fixed group.
For inventory and forecasting data, the unit-level independence assumption concerns dependence from
shared calendar effects as well as entity-specific dynamics. Web Traffic Weekly fits grouping thresholds
to pooled pre-outcome histories of calibration and test pages. We treat this sample-dependent grouping
empirically.

\subsubsection{Statistical Intervals}
\label{app:statistical-units}

The phase-transition study uses 500 independent repetitions, and the summary non-identification study uses
5,000 per sample size. Bootstrap intervals are paired, two-sided percentile intervals with empirical
0.025 and 0.975 endpoints. For unhealthy inventory,
each SKU-level shared-minus-persistence contrast is first averaged across the three fixed training seeds.
5,000 bootstrap resamples then draw the 33 SKUs. Inventory intervals are reported in
Section~\ref{app:inventory-robustness}.
Monash intervals use 10,000 bootstrap samples over 2,674 series, with MASE restricted to
the 2,504 series having a positive training scale. M4 uses 10,000 resamples of held-out series in each
frequency, and Web Traffic Weekly uses 5,000 resamples of its 87,222 test series.

FourRooms overall prediction intervals resample 118 complete episodes and recompute
the transition-weighted ratio estimator in each of 10,000 bootstrap replicates. Groupwise gain intervals
use 10,000 paired resamples of the episodes containing the group, averaging their persistence-minus-candidate
loss differences with equal episode weights.

MuJoCo reports the sample standard deviation across three fixed
training seeds for each of nine dataset variants. Its groupwise bounded-gain intervals use 10,000 paired resamples of represented episodes,
with the same equal-episode weighting as calibration. Each interval is computed for one group in one
run. For MuJoCo policy-loss contrasts, the three training seeds are averaged within each test episode.
The 10,000 bootstrap replicates resample episodes separately within each dataset and then equally
weight the nine dataset means, conditional on the fitted models and calibration decisions.
\FloatBarrier

\subsection{Decision Diagnostics}

We examine selection at the movement boundary, decision recovery from loss gains, and the calibration
sample sizes needed for certification.

\subsubsection{Confidence correction at the movement boundary}
\label{app:phase-confidence}

The phase-transition population has zero-change probability $q$ and a unit-change proposal, with
population gain $1-2q$. For $n$ independent calibration observations, the number of unchanged outcomes
satisfies $K\sim\mathrm{Binomial}(n,q)$. The empirical-sign rule accepts when $1-2K/n>0$, and the
LCB rule accepts when $1-2K/n>\sqrt{2\log(G/\delta)/n}$.
Table~\ref{tab:phase-confidence} evaluates these probabilities exactly for the same $G=11$ groups,
$q\in\{0.1,0.2,0.3,0.4,0.45,0.5,0.55,0.6,0.7,0.8,0.9\}$, and $\delta=0.05$ used in
Figure~\ref{fig:phase}. Independence across groups gives the probability of accepting at least one
harmful group as $1-\prod_{q>0.5}(1-a_q)$, where $a_q$ is its acceptance probability.
Harmful groups have strictly negative population gain. The group at $q=0.5$ has zero gain.

\begin{table}[ht]
\centering
\small
\setlength{\tabcolsep}{4pt}
\caption{Exact selection probabilities and population regret at the movement boundary.
Harmful acceptance is the probability of accepting at least one harmful group.
Power averages acceptance over beneficial groups. Coverage and regret equally weight all 11 groups.}
\label{tab:phase-confidence}
\begin{tabular}{llS[table-format=1.4e-2]S[table-format=3.1,table-space-text-post={\%}]S[table-format=3.1,table-space-text-post={\%}]S[table-format=1.6]}
\toprule
$n$ & Rule & \multicolumn{1}{c}{Harmful acceptance} & \multicolumn{1}{c}{Power} & \multicolumn{1}{c}{Coverage} & \multicolumn{1}{c}{Regret} \\
\midrule
50 & Sign & 2.4341e-1 & 92.3\% & 48.3\% & 0.007310 \\
50 & LCB & 3.1350e-5 & 45.0\% & 20.5\% & 0.057217 \\
200 & Sign & 6.9645e-2 & 98.2\% & 49.5\% & 0.001504 \\
200 & LCB & 1.0008e-6 & 66.6\% & 30.3\% & 0.021654 \\
1,000 & Sign & 6.8081e-4 & 100.0\% & 49.9\% & 0.000014 \\
1,000 & LCB & 6.3507e-11 & 89.2\% & 40.6\% & 0.004904 \\
\bottomrule
\end{tabular}
\end{table}

\subsubsection{Numerical illustration of summary non-identification}
\label{app:summary-nonidentification-numerical}

We fix $p=0.65$, $\alpha=1$, $\beta=\alpha/\sqrt{1-p}$, and offer the same candidate correction $\beta$
under both laws in Theorem~\ref{thm:summary-nonidentification}. Loss is absolute error divided by
$\alpha+\beta$.
The candidate's population gain over persistence is $-0.3867$ under $P_0$ and $0.1885$ under $P_1$,
so an oracle given the law identity rejects under $P_0$ and accepts under $P_1$. In contrast, occurrence probability,
Bernoulli entropy, and conditional variance are respectively $0.65$, $0.6474$, and $0.65$ under both
laws. Even the best oracle rule restricted to these identical summaries must take one shared action, so it
persists and incurs $0.0942$ excess loss under an equally weighted mixture. Table~\ref{tab:summary-nonidentification} compares empirical-sign and LCB decision recovery over
5,000 repetitions per sample size, using $n$ calibration observations per law, $G=2$, and
$\delta=0.05$. Both rules recover the opposite decisions as sample size increases.

\begin{table}[ht]
\centering
\small
\setlength{\tabcolsep}{5pt}
\caption{Decision recovery under the two laws. Recovery is the fraction of repetitions with both
decisions correct. Excess loss is relative to an oracle given the law identity under the equal mixture.}
\label{tab:summary-nonidentification}
\begin{tabular}{llS[table-format=1.4]S[table-format=1.4]}
\toprule
Rule & $n$ & \multicolumn{1}{c}{Both decisions correct $\uparrow$} & \multicolumn{1}{c}{Normalized excess loss $\downarrow$} \\
\midrule
Best identical-summary oracle & --- & 0.0000 & 0.0942 \\
Empirical sign & 25 & 0.9422 & 0.0054 \\
Empirical sign & 200 & \bfseries 1.0000 & \bfseries 0.0000 \\
LCB gain & 25 & 0.0002 & 0.0942 \\
LCB gain & 100 & 0.0870 & 0.0860 \\
LCB gain & 200 & 0.4670 & 0.0502 \\
LCB gain & 500 & 0.9926 & 0.0007 \\
LCB gain & 1,000 & \bfseries 1.0000 & \bfseries 0.0000 \\
\bottomrule
\end{tabular}
\end{table}
\FloatBarrier

\subsubsection{Calibration Counts and Certification Thresholds}
\label{app:certification-power}

Table~\ref{tab:certification-power} separates the empirical acceptance condition $\widehat\mu_g>r_g$
from the sufficient population-gain condition $\mu_g>2r_g$ in
Proposition~\ref{prop:certification-power-main}. The finer Monash partitions combine larger radii with
changes in candidate gain (see Section~\ref{app:monash-sensitivity}). M4's ambiguous group has positive
mean gain below its radius.

\begin{table}[ht]
\centering
\small
\setlength{\tabcolsep}{2.5pt}
\caption{Calibration counts, mean gains, and Hoeffding radii. The $2r_g$ column gives the sufficient
population-gain threshold for high-probability acceptance. Monash rows use the equal-count sensitivity
partitions in Section~\ref{app:monash-sensitivity}. For these rows, radii are group medians, and gain and count intervals
report ranges.}
\label{tab:certification-power}
\begin{tabular}{llllS[table-format=1.4]S[table-format=1.4]c}
\toprule
Setting & Group & $n_g$ & $\widehat\mu_g$ & \multicolumn{1}{c}{$r_g$} & \multicolumn{1}{c}{$2r_g$} & \multicolumn{1}{c}{Accepted} \\
\midrule
Monash, $G=2$ & range & 1,337 & [0.0929,0.1164] & 0.0743 & 0.1486 & 2/2 \\
Monash, $G=4$ & range & 668--669 & [-0.1566,0.1051] & 0.1145 & 0.2290 & 0/4 \\
Monash, $G=8$ & range & 334--335 & [-0.0571,0.1384] & 0.1743 & 0.3487 & 0/8 \\
M4, $G=3$ & Candidate-favored & 2,568 & 0.0920 & 0.0565 & 0.1129 & yes \\
M4, $G=3$ & Ambiguous & 1,518 & 0.0192 & 0.0734 & 0.1469 & no \\
M4, $G=3$ & Persistence-favored & 15,230 & -0.0487 & 0.0232 & 0.0464 & no \\
\bottomrule
\end{tabular}
\end{table}
\FloatBarrier

\subsection{Selector and Grouping Analyses}

We compare selection rules for fixed candidates and examine sensitivity to grouping and candidate fitting.

\subsubsection{Selector comparison with fixed candidates}
\label{app:ablation}

Table~\ref{tab:decision-ablation} compares the empirical-sign and LCB rules with fixed candidates and
groups. The sign rule accepts a group when its calibration mean gain is positive. In M4 Monthly,
Quarterly, FourRooms, and short-horizon MuJoCo, confidence correction reduces coverage by excluding
additional groups accepted by the sign rule. These groups have positive test gains. Both rules make
the same decisions on Monash, M4 Daily, and Web Traffic.

\begin{table}[!htb]
\centering
\small
\setlength{\tabcolsep}{3.5pt}
\caption{Confidence correction with fixed candidates (loss $\downarrow$).
$\Delta$ is LCB-minus-sign loss. Coverage counts series for forecasting,
transitions for FourRooms, and prediction windows for MuJoCo. MuJoCo coverage averages seeds and
dataset variants. Losses are compared within rows.}
\label{tab:decision-ablation}
\begin{tabular}{lS[table-format=2.4]S[table-format=2.4]S[table-format=+1.4]S[table-format=3.1,table-space-text-post={\%}]S[table-format=3.1,table-space-text-post={\%}]}
\toprule
Setting & \multicolumn{1}{c}{Sign} & \multicolumn{1}{c}{LCB} & \multicolumn{1}{c}{$\Delta$} & \multicolumn{1}{c}{Sign cov.} & \multicolumn{1}{c}{LCB cov.} \\
\midrule
Monash (MAE) & 0.3913 & 0.3913 & 0 & 100.0\% & 100.0\% \\
M4 Monthly (bounded) & 0.5860 & 0.5878 & +0.0019 & 22.2\% & 14.0\% \\
M4 Quarterly (bounded) & 0.5789 & 0.5871 & +0.0082 & 17.6\% & 0.0\% \\
M4 Daily (bounded) & 0.6198 & 0.6198 & 0 & 0.0\% & 0.0\% \\
Web Traffic (bounded) & 0.3031 & 0.3031 & 0 & 0.0\% & 0.0\% \\
FourRooms (loss) & 0.0943 & 0.1220 & +0.0277 & 100.0\% & 20.4\% \\
MuJoCo, $h=1$ (clipped) & 0.0157 & 0.0669 & +0.0512 & 100.0\% & 33.3\% \\
MuJoCo, $h=5$ (clipped) & 0.0435 & 0.0730 & +0.0295 & 100.0\% & 86.2\% \\
MuJoCo, $h=10$ (clipped) & 0.0694 & 0.0694 & 0 & 100.0\% & 100.0\% \\
MuJoCo, $h=20$ (clipped) & 0.1224 & 0.1224 & 0 & 100.0\% & 100.0\% \\
\bottomrule
\end{tabular}
\end{table}

The MuJoCo comparison uses the 27 repeated runs in Section~\ref{app:mujoco-bounded-gains}, with
clipped losses and aggregation defined in Sections~\ref{app:metric-definitions}
and~\ref{app:statistical-units}. Positive calibration mean gains in every group make the sign rule
identical to always execute.
LCB-minus-sign clipped-loss differences at $h=1,5$ have paired 95\% intervals
$[0.0501,0.0526]$ and $[0.0287,0.0303]$, respectively. M4 Monthly's ambiguous group has positive test
gain $0.0227$ with interval $[0.0177,0.0276]$, accounting for the sign rule's additional improvement.

\FloatBarrier

\subsubsection{Monash grouping sensitivity}
\label{app:monash-sensitivity}

This post-hoc analysis varies both the grouping and its fitted candidates. For
$K_{\mathrm{grp}}\in\{1,2,4,8\}$, series are ordered by training zero fraction, with identifier-based
tie breaking, and divided into equal-count strata. A separate empirical-median candidate is fitted from
the training observations in each stratum. Table~\ref{tab:monash-grouping} compares always executing,
empirical-sign selection, and Hoeffding selection for each fitted configuration. At four and eight strata,
candidate gains have mixed signs and the empirical-sign rule improves MAE over
always executing. Hoeffding selection rejects every group: all calibration mean gains lie below their radii
(see Table~\ref{tab:certification-power}).

\begin{table}[ht]
\centering
\small
\setlength{\tabcolsep}{4pt}
\caption{Monash sensitivity to grouping and candidate fitting (test MAE $\downarrow$).
``Mixed'' indicates positive and negative stratum-level test gains within the same configuration.
Candidates are refitted for each $K_{\mathrm{grp}}$. Bold marks row minima.}
\label{tab:monash-grouping}
\begin{tabular}{lS[table-format=1.3]S[table-format=1.3]S[table-format=1.3]S[table-format=1.3]S[table-format=3.1,table-space-text-post={\%}]S[table-format=3.1,table-space-text-post={\%}]c}
\toprule
$K_{\mathrm{grp}}$ & \multicolumn{1}{c}{Persist.} & \multicolumn{1}{c}{Candidate} & \multicolumn{1}{c}{Sign gate} & \multicolumn{1}{c}{Hoeffding} & \multicolumn{1}{c}{Sign cov.} & \multicolumn{1}{c}{Hoeffding cov.} & \multicolumn{1}{c}{Mixed?} \\
\midrule
1 & 0.573 & \bfseries 0.391 & \bfseries 0.391 & \bfseries 0.391 & 100.0\% & 100.0\% & no \\
2 & 0.573 & \bfseries 0.391 & \bfseries 0.391 & \bfseries 0.391 & 100.0\% & 100.0\% & no \\
4 & 0.573 & 0.492 & \bfseries 0.446 & 0.573 & 75.0\% & 0.0\% & yes \\
8 & 0.573 & 0.429 & \bfseries 0.420 & 0.573 & 87.5\% & 0.0\% & yes \\
\bottomrule
\end{tabular}
\end{table}
\FloatBarrier

\subsection{Inventory Results Across Metrics and Time}
\label{app:inventory-robustness}

The shared inventory candidate has higher mean error than persistence across raw MAE,
inventory-normalized MAE, and inventory WAPE (see Table~\ref{tab:inventory-robustness}). In the main
evaluation, the paired 95\% bootstrap intervals for shared-minus-persistence raw MAE at $h=1,7,14,30$
are $[-0.0006,0.0061]$, $[0.0005,0.0187]$, $[0.0007,0.0309]$, and $[0.0002,0.0512]$, respectively.
Only the $h=1$ interval includes zero. Section~\ref{app:statistical-units} specifies the bootstrap procedure.
The same error ordering holds at a later historical forecast origin, using only earlier transitions for training.

\begin{table}[ht]
\centering
\small
\setlength{\tabcolsep}{2.5pt}
\caption{Inventory prediction across scales and time (all metrics $\downarrow$).
P denotes persistence and S the three-seed mean of the shared ensemble.
The later-window column evaluates a single historical forecast origin. Bold marks each pair's lower loss.}
\label{tab:inventory-robustness}
\begin{tabular}{l@{\quad}S[table-format=2.3]S[table-format=2.3]@{\quad}S[table-format=1.5]S[table-format=1.5]@{\quad}S[table-format=1.5]S[table-format=1.5]@{\quad}S[table-format=1.3]S[table-format=1.3]}
\toprule
& \multicolumn{2}{c}{Raw MAE} & \multicolumn{2}{c}{Normalized MAE}
& \multicolumn{2}{c}{WAPE} & \multicolumn{2}{c}{Later-window MAE} \\
\cmidrule(lr){2-3}\cmidrule(lr){4-5}\cmidrule(lr){6-7}\cmidrule(lr){8-9}
$h$ & \multicolumn{1}{c}{P} & \multicolumn{1}{c}{S} & \multicolumn{1}{c}{P} & \multicolumn{1}{c}{S} & \multicolumn{1}{c}{P} & \multicolumn{1}{c}{S} & \multicolumn{1}{c}{P} & \multicolumn{1}{c}{S} \\
\midrule
1  & \bfseries 1.027 & 1.030 & \bfseries 0.00864 & 0.00882 & \bfseries 0.00874 & 0.00876 & \bfseries 0.293 & 0.295 \\
7  & \bfseries 3.958 & 3.967 & \bfseries 0.03108 & 0.03145 & \bfseries 0.03366 & 0.03373 & \bfseries 1.356 & 1.364 \\
14 & \bfseries 6.837 & 6.850 & \bfseries 0.06226 & 0.06268 & \bfseries 0.05804 & 0.05815 & \bfseries 2.420 & 2.428 \\
30 & \bfseries 13.737 & 13.759 & \bfseries 0.11348 & 0.11404 & \bfseries 0.11667 & 0.11685 & \bfseries 7.365 & 7.385 \\
\bottomrule
\end{tabular}
\end{table}

\FloatBarrier

\subsection{Additional Benchmark Results}

We report additional forecasting evaluations, FourRooms groupwise gains and closed-loop results, and
MuJoCo groupwise bounded-loss gains.

\subsubsection{Cross-frequency and external evaluation}
\label{app:cross-frequency-diagnostics}

Table~\ref{tab:cross-frequency} reports the Monthly result and subsequent forecasting evaluations.
Quarterly \citep{makridakis2018m4} uses the locally frozen Hoeffding protocol and rejects all groups. The pre-specified
empirical-sign comparator accepts the candidate-favored group and improves on both extreme policies.
Its bounded-loss difference against persistence is $-0.00817$, with a paired 95\% bootstrap interval
of $[-0.00910,-0.00728]$. Against always accepting, the difference is $-0.03844$, with an interval
of $[-0.03982,-0.03702]$.

Quarterly calibration data identify the role of the confidence radius. For improvement
$Z\in[-B,B]$, unbiased sample variance
$\widehat V_g=(\sum_{i:g(X_i)=g}(Z_i-\widehat\mu_g)^2)/(n_g-1)$ with $n_g\geq2$, and $G$
predeclared groups, the simultaneous
empirical-Bernstein radius is
\[
r_g^{\mathrm{EB}}=
\sqrt{\frac{2\widehat V_g\log(2G/\delta)}{n_g}}
+\frac{14B\log(2G/\delta)}{3(n_g-1)},
\]
obtained by applying Theorem~4 of \citet{maurer2009empirical} to $(B-Z)/(2B)\in[0,1]$
with failure probability $\delta/G$, then taking a union bound over groups. Groups with fewer than two calibration units receive persistence.
In the Quarterly candidate-favored group, $n_g=1{,}540$, the mean gain is
$0.0502$, and the sample variance is $0.0169$. The Hoeffding and empirical-Bernstein LCBs are respectively
$-0.0227$ and $0.0254$.

The variance-adaptive rule, split, groups, and criterion of improvement over both baselines were locally
frozen before downloading M4 Daily. Daily rejects every group: the calibration counts are
1/19/1,640 for candidate-favored/ambiguous/persistence-favored, and all three mean gains are nonpositive.
The selector matches persistence, which has lower test loss than always executing the candidate.

The Web Traffic Weekly protocol was also locally frozen before downloading the data
\citep{godahewa2021monash}. The annual-seasonal proposal was evaluated on 57,841 calibration and 87,222
held-out test page series. The three calibration groups had mean gains of
$-0.0408$, $-0.0840$, and $-0.2092$, with empirical-Bernstein radii below $0.0065$. Hence every gate
rejected, and selective execution matched persistence. Its paired 95\% bootstrap interval for the
loss difference against always executing was $[-0.1120,-0.1092]$.

\begin{table}[ht]
\centering
\small
\setlength{\tabcolsep}{3.5pt}
\caption{Forecasting replication results (bounded loss $\downarrow$). Q-Sign is the pre-specified
Quarterly empirical-sign comparator. EB is the empirical-Bernstein gate. The final column indicates
lower test loss than both persistence and always executing. Bold marks row minima.}
\label{tab:cross-frequency}
\begin{tabular}{lS[table-format=1.3]S[table-format=1.3]S[table-format=1.3]S[table-format=3.1,table-space-text-post={\%}]c}
\toprule
Setting & \multicolumn{1}{c}{Persist.} & \multicolumn{1}{c}{Always} & \multicolumn{1}{c}{Selected} & \multicolumn{1}{c}{Coverage} & \multicolumn{1}{c}{Beats both} \\
\midrule
M4 Monthly, Hoeffding & 0.599 & 0.621 & \bfseries 0.588 & 14.0\% & Yes \\
M4 Quarterly, Hoeffding & \bfseries 0.587 & 0.617 & \bfseries 0.587 & 0.0\% & No \\
M4 Quarterly, Q-Sign & 0.587 & 0.617 & \bfseries 0.579 & 17.6\% & Yes \\
M4 Daily, EB & \bfseries 0.620 & 0.672 & \bfseries 0.620 & 0.0\% & No \\
Web Traffic Weekly, EB & \bfseries 0.303 & 0.414 & \bfseries 0.303 & 0.0\% & No \\
\bottomrule
\end{tabular}
\end{table}
\FloatBarrier

\subsubsection{FourRooms groupwise gains and closed-loop evaluation}
\label{app:fourrooms-closed-loop}
Table~\ref{tab:fourrooms-group-gain} evaluates the fixed candidate under the episode--group loss used
for calibration. Both groups have positive test gains, with calibration selecting the turn group.

\begin{table}[ht]
\centering
\small
\setlength{\tabcolsep}{5pt}
\caption{FourRooms groupwise bounded-loss gain over persistence. Test gains equally weight episodes
containing the group. Intervals are paired 95\% episode bootstrap intervals for each group.}
\label{tab:fourrooms-group-gain}
\begin{tabular}{lcS[table-format=-1.4]rS[table-format=1.4]l}
\toprule
Group & \multicolumn{1}{c}{Decision} & \multicolumn{1}{c}{Cal.\ LCB} & \multicolumn{1}{c}{Test episodes} & \multicolumn{1}{c}{Test gain} & 95\% interval \\
\midrule
Turn & Execute & 0.0291 & 117 & 0.2964 & $[0.2625,0.3292]$ \\
Forward & Persist & -0.2184 & 118 & 0.0329 & $[0.0287,0.0372]$ \\
\bottomrule
\end{tabular}
\end{table}

A shared controller scores proposed next observations using a training-only nearest-neighbor estimate
of remaining expert steps. Across 30 evaluation seeds, persistence and LCB execution each succeed in
$2/30$ episodes, and always executing succeeds in $1/30$. LCB execution coverage averages $23.9\%$.
The paired success-rate difference against persistence is $0.000$ with a 95\% interval of
$[-0.100,0.100]$. The pilot shows similarly low navigation success across all three execution rules.

\subsubsection{MuJoCo groupwise bounded-loss gains}
\label{app:mujoco-bounded-gains}

The bounded-loss evaluation repeats the full 27-run configuration with the same datasets, splits,
training seeds, and gate settings, retaining newly fitted models and window-level losses.
All 324 group acceptance decisions match the original runs. The maximum absolute differences in
calibration LCB and unclipped group gain are $0.00163$ and $0.00407$, respectively.
Table~\ref{tab:mujoco-bounded-gains} reports test gains under $L^{\mathrm{MJ}}_{B,ig}$ for both accepted
and rejected groups. All 324 gains have individual 95\% bootstrap intervals with positive lower endpoints.
The seven groups with negative unclipped gains have positive bounded gains ranging from $0.3688$ to $0.7560$.

\begin{table}[ht]
\centering
\small
\setlength{\tabcolsep}{5pt}
\caption{MuJoCo groupwise bounded-loss gains in the 27 repeated runs. Gain equally weights represented
episodes. Positive CI counts groups whose individual 95\% bootstrap interval lies above zero.
Minima are taken within each acceptance category.}
\label{tab:mujoco-bounded-gains}
\begin{tabular}{lcrrS[table-format=1.4]S[table-format=1.4]}
\toprule
$h$ & \multicolumn{1}{c}{Decision} & \multicolumn{1}{c}{Groups} & \multicolumn{1}{c}{Positive CI} & \multicolumn{1}{c}{Min.\ gain} & \multicolumn{1}{c}{Min.\ CI lower} \\
\midrule
1 & Accept & 27 & 27 & 0.7455 & 0.7283 \\
1 & Reject & 54 & 54 & 0.0177 & 0.0166 \\
5 & Accept & 70 & 70 & 0.2509 & 0.2366 \\
5 & Reject & 11 & 11 & 0.1317 & 0.1230 \\
10 & Accept & 81 & 81 & 0.3501 & 0.3305 \\
20 & Accept & 81 & 81 & 0.3688 & 0.3577 \\
\bottomrule
\end{tabular}
\end{table}
\FloatBarrier

\section{Experiment Configurations}
\label{app:reproducibility}

Tables~\ref{tab:repro-configs} and~\ref{tab:repro-control-configs} summarize the experiment configurations.
Architectures and optimizers were fixed before test-result aggregation. Candidate fitting and gate
selection use the training and calibration splits, respectively. Test-oracle rules provide diagnostic
reference values. Appendix~\ref{app:assumption-audit} specifies sampling assumptions, and
Appendix~\ref{app:metric-definitions} defines the losses and aggregation order.

\begingroup
\small
\setlength{\tabcolsep}{3.2pt}
\setlength{\LTcapwidth}{\textwidth}
\setlength{\LTpre}{10pt}
\setlength{\LTpost}{10pt}
% Match the font and six-point lower gap of the other table captions.
\makeatletter
\long\def\LT@makecaption#1#2#3{%
  \LT@mcol\LT@cols c{\hbox to\z@{\hss
    \parbox[t]\LTcapwidth{\normalfont\normalsize #1{#2: }#3\par\vskip6pt}%
  \hss}}}
\makeatother
\begin{longtable}{@{}
>{\raggedright\arraybackslash}p{0.14\textwidth}
>{\raggedright\arraybackslash}p{0.20\textwidth}
>{\raggedright\arraybackslash}p{0.25\textwidth}
>{\raggedright\arraybackslash}p{\dimexpr0.41\textwidth-6\tabcolsep\relax}@{}}
\caption{Forecasting and inventory configurations. Counts refer to calibration and test units where paired.}
\label{tab:repro-configs}\\
\toprule
Experiment & Split / sampling unit & Candidate and optimization & Gate and evaluation \\
\midrule
\endfirsthead
\multicolumn{4}{l}{\small Table~\thetable\ continued.}\\[4pt]
\toprule
Experiment & Split / sampling unit & Candidate and optimization & Gate and evaluation \\
\midrule
\endhead

Controlled & Independent observations within each group, 500 repetitions & Fixed unit correction, seed 20260827 &
$n_g\in\{50,200,1000\}$, $\delta=0.05$, zero mass $0.10$--$0.90$ \\
Summary non-identification & Independent observations from each law, 5,000 repetitions & Common correction $\beta$, seed 20260829 &
$p=0.65$, $\alpha=1$, $n_g\in\{25,50,100,200,500,1000\}$, $\delta=0.05$ \\
Monash & 27/12/12 months, rolling one-step evaluation & Groupwise empirical median of training target values &
Training zero-fraction threshold $0.75$, $B=1$, $\delta=0.05$, minimum 100 series, seasonal period
12 \\
M4 Monthly & Hash-disjoint 40/60\% calibration/test series, 19,316/28,684 & Seasonal persistence, three
groups fixed by train-only backtest ratio & Horizon 18, $B=1$, $\delta=0.05$, minimum 500 series, 10,000
paired series bootstraps \\
M4 Q/D diagnostics & Hash-disjoint complete series, 9,557/14,443 Quarterly and 1,660/2,567 Daily & Seasonal
persistence with periods 4/7, unchanged train-only ratio groups & Horizons 8/14, Quarterly Hoeffding,
Daily empirical-Bernstein, 10,000 paired bootstraps \\
Web Traffic Weekly & Hash-disjoint complete page series, 57,841/87,222 & Annual seasonal persistence,
three groups from pooled pre-outcome backtest-ratio tertiles & Horizon 8, empirical-Bernstein, minimum 1,000
series, 5,000 paired series bootstraps \\
Unhealthy inventory & Entity-grouped 5-fold CV, 33 SKUs & Type-aware probabilistic MLP ensemble with three training seeds &
Horizons $h\in\{1,7,14,30\}$, 5,000 entity bootstraps \\
\bottomrule
\end{longtable}
\endgroup

\clearpage
\begin{table*}[!ht]
\centering
\small
\setlength{\tabcolsep}{4pt}
\caption{Action-conditioned experiment configurations, with complete episodes as the sampling units.}
\label{tab:repro-control-configs}
\begin{tabularx}{\textwidth}{@{}
>{\raggedright\arraybackslash}p{0.13\textwidth}
>{\raggedright\arraybackslash}p{0.22\textwidth}
>{\raggedright\arraybackslash}p{0.31\textwidth}
>{\raggedright\arraybackslash}X@{}}
\toprule
Experiment & Split / sampling unit & Candidate and optimization & Gate and evaluation \\
\midrule
FourRooms & Ordered episodes: 354 training, 118 calibration, 118 test & Empirical action-conditioned transition table,
no neural optimization & Turn/forward groups, $\delta=0.05$, minimum 30 episodes, 10,000 paired episode
bootstraps \\
MuJoCo & Complete episodes, 70/15/15\%, three seeds & 3-member delta-MLP ensemble, two width-256 hidden
layers, SiLU, 8 epochs, batch 4096, AdamW lr $3\!\times\!10^{-4}$, weight decay $10^{-4}$, bootstrap
probability 0.8 & $h\in\{1,5,10,20\}$, horizon $\times$ train-fitted uncertainty tertile, $\delta=0.05$,
minimum 30 episodes \\
FourRooms pilot & 30 paired environment seeds, 2027100--2027129 & Shared controller with 5 nearest neighbors and a
maximum of 100 steps & Same turn/forward gate, 5,000 paired bootstrap resamples \\
\bottomrule
\end{tabularx}
\end{table*}
\FloatBarrier

For MuJoCo, each ensemble member is trained on an independent Bernoulli bootstrap of training
transitions. Normalization statistics and uncertainty tertiles use training units only. MuJoCo uses
training/calibration/test window strides 10/5/5, at most 50,000 train windows per gate horizon, and an
evaluation batch of 8,192. Its training seeds are 42, 20260827, and 314159.

\FloatBarrier

\end{document}